\documentclass[10pt,letterpaper]{article}
\usepackage[utf8]{inputenc}
\usepackage{amsthm,amssymb,latexsym,amsmath,graphics,graphicx,mathtools,subcaption}
\usepackage[usenames,dvipsnames]{color}
\usepackage[shortlabels]{enumitem}
\usepackage[linesnumbered,ruled,vlined]{algorithm2e}
\SetKwComment{Comment}{(}{)}
\usepackage{hyperref}
\hypersetup{
  colorlinks   = true,    
  urlcolor     = blue,    
  linkcolor    = blue,    
  citecolor    = red      
}
\theoremstyle{plain}
\newtheorem{theorem}{Theorem}[section]

\newtheorem{corollary}{Corollary}[section]
\newtheorem{lemma}{Lemma}[section]
\newtheorem{proposition}{Proposition}[section]
\newtheorem{definition}{Definition}[section]

\theoremstyle{definition}

\newtheorem{example}{Example}[section]

\newtheorem{remark}{Remark}[section]

\usepackage[left=2cm,right=2cm,top=2cm,bottom=2cm]{geometry}

\def\XX{{\mathbb X}}
\def\BB{{\mathbb B}}

\def\TT{{\mathbb T}}
\def\RR{{\mathbb R}}

\def\ZZ{{\mathbb Z}}

\def\MM{\mathbb{M}}
\def\PPI{{{\rm I}\kern-1pt\Pi}}
\def\SS{{\mathbb S}}

\newcommand\abs[1]{\left|#1\right|}

\newcommand\dist{\operatorname{dist}}
\newcommand\argmax{\operatorname*{argmax}}

\newcommand{\be}{\begin{equation}}
\newcommand{\ee}{\end{equation}}
\newcommand{\ba}{\begin{aligned}}
\newcommand{\ea}{\end{aligned}}
\newcommand{\ben}{\begin{enumerate}}
\newcommand{\een}{\end{enumerate}}
\newcommand{\bit}{\begin{itemize}}
\newcommand{\eit}{\end{itemize}}
\def\ls{\lesssim}
\def\D{{\mathcal D}}
\def\gs{\gtrsim}
\def\disp{\displaystyle}

\usepackage{nomencl}
\makenomenclature

\newcommand{\yadi}{\nomenclature}

\title{A signal separation view of classification}
\author{H. N. Mhaskar\thanks{
Institute of Mathematical Sciences, Claremont Graduate University, Claremont, CA 91711. 
\textsf{email:} hrushikesh.mhaskar@cgu.edu.
The research is  supported in part by  ONR grants N00014-23-1-2394, N00014-23-1-2790.} \and Ryan O'Dowd\thanks{Institute of Mathematical Sciences, Claremont Graduate University, Claremont, CA 91711. 
\textsf{email:} ryan.o'dowd@cgu.edu.}}
\date{\today}

\numberwithin{equation}{section}

\begin{document}
\maketitle


\begin{abstract}
The problem of classification in machine learning has often been approached in terms of function approximation. 
In this paper, we propose an alternative approach for classification in arbitrary compact metric spaces which, in theory, yields both the number of classes, and a perfect classification using a minimal number of queried labels.
Our approach uses localized trigonometric polynomial kernels initially developed for the point source signal separation problem in signal processing. 
Rather than point sources, we argue that the various classes come from different probability measures.
The localized kernel technique developed for separating point sources is then shown to separate the supports of these measures.
This is done in a hierarchical manner in our MASC algorithm to accommodate touching/overlapping class boundaries.
We illustrate our theory on several simulated and real life datasets, including the Salinas and Indian Pines hyperspectral datasets and a document dataset.
\end{abstract}

\section{Introduction}
\label{sec:intro}
A fundamental problem in machine learning is the following. Let $\{(x_j,y_j)\}_{j=1}^M$ be random samples from an \textbf{unknown} probability measure $\tau$. 
The problem is to approximate the conditional expectation $f(x)=\mathbb{E}_\tau(y|x)$ as a function of $x$. 
Naturally, there is a huge amount of literature studying function approximation by commonly used tools in machine learning such as neural and kernel based networks. 
For example, the universal approximation theorem gives conditions under which a neural network can approximate an arbitrary \textbf{continuous} function on an arbitrary compact subset of the ambient Euclidean space \cite{chuili1992,cybenko-net}.
The estimation of the complexity of the approximation process typically assumes some smoothness conditions on $f$, examples of which include, the number of derivatives, membership in various classes such as Besov spaces, Barron spaces, variation spaces, etc \cite{barron-universal,mhaskar-multilayer,mhaskar-net}.

A very important problem is one of classification.
Here the values of $y_j$ can take only finitely many (say $K$) values, known as the class labels. 
In this case, it is fruitful to approximate the classification function, defined by $f(x)=\argmax_k \mathsf{Prob}(k|x)$ \cite{murphy2022probabilistic}. 
Obviously, this function is only piecewise continuous, so that the universal approximation theorem does not apply directly. 
In the case when the classes are supported on well-separated sets, one may refer to extension theorems such as Stein extension theorems \cite{stein} in order to justify the use of the various approximation theorems to this problem.

While these arguments are sufficient for pure existence theorems, they also create difficulties in an actual implementation, in particular, because these extensions are not easy to construct. 
In fact, this would be impossible if the classes are not well-separated, and might even overlap.
Even if the classes are well-separated, and  each class represents a Euclidean domain, any lack of smoothness in the boundaries of these domains is a problem.
Some recent efforts, for example,  by Petersen and Voigtl\"ander  \cite{petersen-relu} deal with the question of accuracy in approximation when the class boundaries are not smooth. 
However, a popular assumption in the last twenty years or so is that the data is distributed according a probability measure supported on a low dimensional manifold of a high dimensional ambient Euclidean space.
In this case, the classes have boundary of measure $0$ with respect to the Lebesgue measure on the ambient space.
Finally, approximation algorithms, especially with deep networks, utilize a great deal of labeled data.

In this paper, we propose a different approach as advocated in \cite{cloningercluster}. 
Thus, we do not assume that $\mathsf{Prob}(k|x)$ is a function, but assume instead that the points $x_j$ in class $k$ comprise the support of a probability measure $\mu_k$. 
The marginal distribution $\mu$ 
\yadi{${\mu}$}{Data measure of the input} 
of $\tau$ along $x$ is then a convex combination of the measures $\mu_k$\yadi{$\mu_k$}{Probability measure for class $k$}.
The fundamental idea is to determine the \textbf{supports} of the measures $\mu_k$ rather than approximating $\mu_k$'s themselves\footnote{If $\nu$ is a positive measure on a metric space $\MM$, we define the support of  any positive measure $\nu$ by $\mathsf{supp}(\nu)=\{x\in\MM, \nu(\BB(x,r))>0 \mbox{ for all } r>0\}$, where $\BB(x,r)$ is the ball of radius $r$ centered at $x$.}.  \yadi{$\mathsf{supp}$}{Support of a measure}
This is done in an unsupervised manner, based only on the $x_j$'s with no label information.
Having done so, we may then query an oracle for the label of one point in the support of each measure, which is then necessarily the label for every other point in the support. 
Thus, we aim to achieve in theory a  perfect classification using a minimal amount of judiciously chosen labeled data.

In order to address the problem of overlapping classes, we take a hierarchical multiscale approach motivated by a paper \cite{dasgupta2010} of Chaudhury and Dasgupta.
Thus, for each value $\eta$ of the minimal separation among classes, we assume that the support of $\mu$ is a disjoint union of $K_\eta$ subsets, each representing one of $K_\eta$ classes, leaving an extra set, representing the overlapping region. 
When we decrease $\eta$, we may eventually capture all the classes, leaving only a negligible overlapping region (ideally with $\mu$-probability $0$).

In \cite{cloningercluster}, we explored a new insight that the problem is analogous to the problem of point source signal separation. 
If each $\mu_k$ were a Dirac delta measure supported at say $\omega_k$, the point source signal separation problem is to find these point sources from finitely many observations of the Fourier transform of $\mu$. 
In the classification problem we do not have point sources and the information comprises samples from $\mu$ rather than its Fourier transform.
Nevertheless, we observed in \cite{cloningercluster} that the techniques developed for the point source signal separation problem can be adapted to the classification problem viewed as the the problem of separation of the supports of $\mu_k$.
In that paper, we assumed only that the data is supported on a compact subset of a Eulidean space, and used a specially designed localized kernel based on Hermite polynomials \cite{chuigaussian} for this purpose. 
Since Hermite polynomials are intrinsically defined on the whole Euclidean space, this creates both numerical and theoretical difficulties.
In this paper, we allow the data to come from an arbitrary compact metric space, and use localized trigonometric polynomial kernels instead.
We feel that this leads to a more satisfactory theory, although one of the accomplishments of this paper is to resolve the technical difficulties required to achieve this generalization.

Our work belongs to the general theory of active learning. In the active learning paradigm for machine learning, one is only given $\mathcal{D}\coloneqq \{x_j\}_{j=1}^M$. However, for any $x_j$ one is allowed to query an oracle for the true value $y_j$ associated with it. The understanding is that this process is costly, so one only wants to query as few points as possible to accurately attain the $y$ values for the rest of the data set. In this way, active learning exists as a bridge between unsupervised learning (where no $y_j$ values are available) and semi-supervised learning (where the $y_j$ values are available only on a preselected set of points $x_j$). The critical problem of active learning is to decide on which points to query. One wishes to query points that give as much information on the rest of the data as possible. We list the survey by Tharwat and Schenck \cite{tharwat-survey} as a resource for recent developments in active learning.

Our main theorem \ref{thm:class_separation} concerns a method to estimate, based on the finite data $\mathcal{D}$, the supports of a measure corresponding to different class labels. To aid in this estimation, we assume that the data lies on an unknown subset of a known compact metric space $\mathbb{M}$ with metric $\rho$. The idea of the theorem is shown as a simple visual in Figure~\ref{fig:thmviz}. In this figure, we see that the true classes (shown at the top) have no minimal separation. Supposing we can define partitions $\mathbf{S}_{1,\eta},\mathbf{S}_{2,\eta},\mathbf{S}_{3,\eta}$ where $\mathbf{S}_{3,\eta}$ is small ($\mu(\mathbf{S}_{3,\eta})\to 0$ as $\eta\to 0$) and $\mathbf{S}_{1,\eta},\mathbf{S}_{2,\eta}$ correspond to the original class labels and have separation $2\eta$, then our theorem gives conditions (based on some parameters $n,\Theta$ and size of the set $\mathcal{D}$) for when our support estimation sets $\mathcal{G}_{1,\eta,n}(\Theta),\mathcal{G}_{2,\eta,n}(\Theta)$ have separation $\eta$ and closely estimate $\mathbf{S}_{1,\eta},\mathbf{S}_{2,\eta}$. In this paper, we consider measures which satisfy this partitioning property for any sufficiently small $\eta$. This allows us to take $\eta\to 0$ so that $\mathbf{S}_{1,\eta},\mathbf{S}_{2,\eta}$ approach the true classes. This framework  allows us to consider the machine learning classification problem even for classes that may have no minimal separation.

\begin{figure}[!ht]
\begin{center}
\includegraphics[width=.6\textwidth]{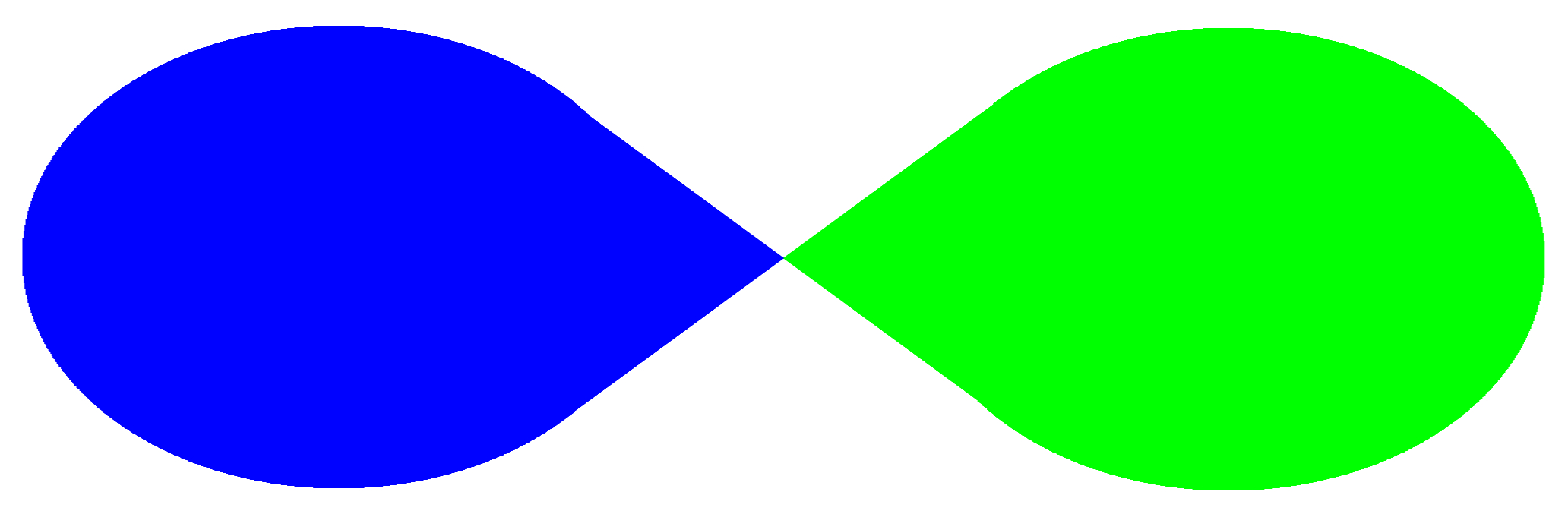}\\
\includegraphics[width=.7\textwidth]{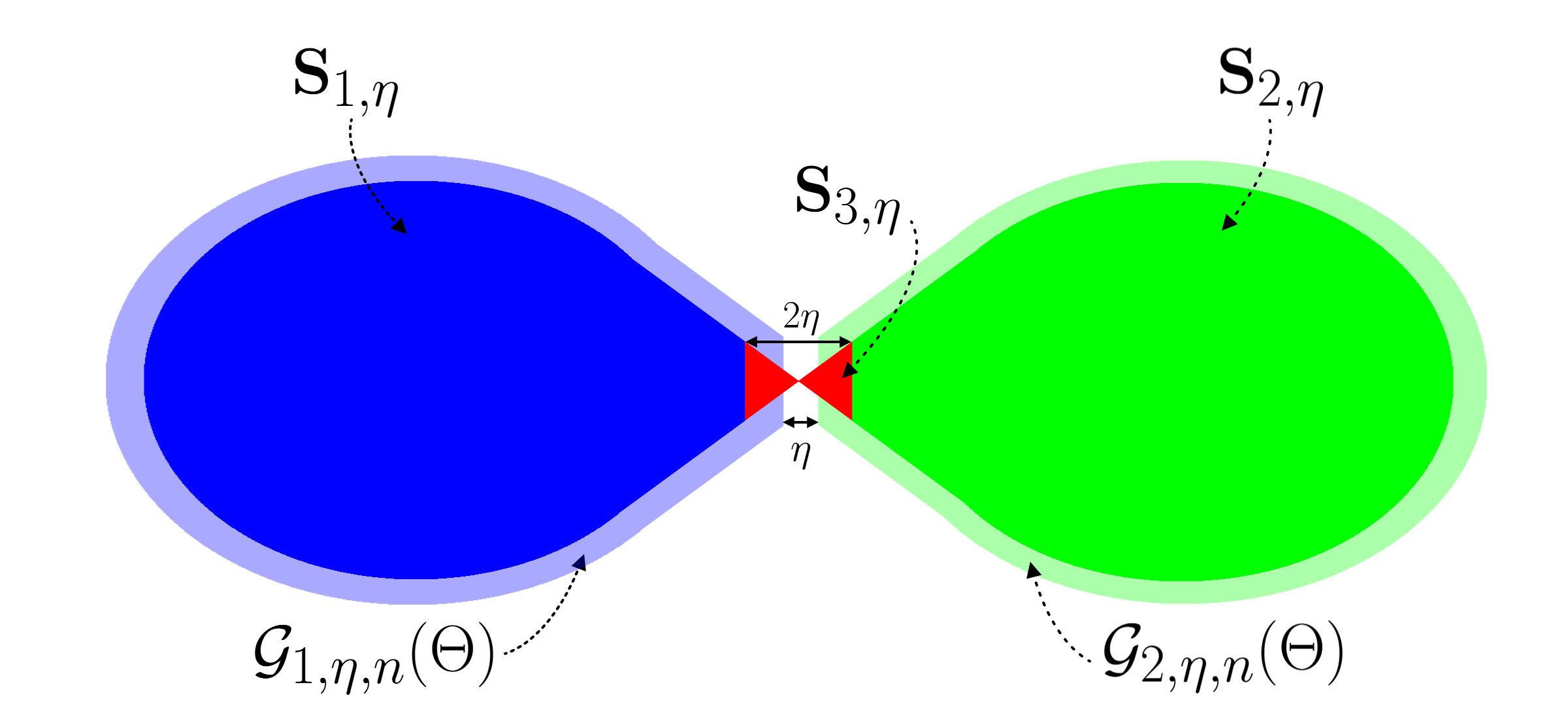}
\end{center}
\caption{Visualization of our main theorem. Top: Supports of two classes with no minimal separation. Bottom: The two classes are separated into sets $\mathbf{S}_{1,\eta},\mathbf{S}_{2,\eta}$ (blue and green) with separation $2\eta$ by removing a remainder set $\mathbf{S}_{3,\eta}$ (red). Our theorem gives conditions for when our support estimation sets $\mathcal{G}_{1,\eta,n}(\Theta),\mathcal{G}_{1,\eta,n}$ (light blue and light green) have separation $\eta$ and are close estimations of $\mathbf{S}_{1,\eta},\mathbf{S}_{2,\eta}$ respectively.}
\label{fig:thmviz}
\end{figure}

Our Algorithm~\ref{alg:MASC}, which has computational improvements and demonstrated accuracy improvements from our prior work in \cite{cloningercluster}, uses the theorem as a starting point to classify data sets in a multiscale active learning fashion. We start by thresholding away low-density points using a measure support estimator funtion (defined in \eqref{eq:support_estimator_def}). Then, we iteratively create successively larger discconnected graphs among the high-density points to extend queried labels in a cautious manner. By ``cautious", we mean that the process is halted locally in the event of conflicting labels belonging to a single graph component. For any graph component where we have not yet queried a point for a label, we do so by choosing the point in the component which maximizes the same support estimator function.

To summarize, the main accomplishments of this paper are:
\begin{itemize}

\item We deal with the classification of data coming from an arbitrary metric space with no further structure, such as the manifold structure.
\item We provide a unified approach to signal separation problems and classification problems.
\item Our results suggest a multiscale approach which does not assume any constraints on class boundaries, including that the classes not overlap.
\item In theory, the number of classes at each scale is an output of the theorem rather than a prior assumption.
\item We develop an algorithm to illustrate the theory, especially in the context of active learning on hyperspectral imaging data.
\end{itemize}

In Section~\ref{sec:related}, we review some literature in this area which is somewhat related to the present work.
In Section~\ref{sec:signals}, we give a brief discussion of the point source signal separation problem and the use of localized trigonometric polynomial kernels to solve it.
In Section~\ref{sec:background}, we describe the background needed to formulate our theorems, which are given in Section~\ref{sec:mainresults}.
The algorithm MASC to implement these results in practice is given in Section~\ref{sec:alg}, and illustrated in the context of a simulated circle and ellipse data set, a document dataset, and two hyperspectral datasets.
The proofs of the results in Section~\ref{sec:mainresults} are given in Section~\ref{sec:proofs}.

\section{Related works}\label{sec:related}

Perhaps the most relevant work to this paper is that of \cite{cloningercluster}. 
That paper also outlines the theory and an algorithm for a classification procedure using a thresholding set based on a localized kernel. There are three major improvements we have made relative to that work in this paper:
\ben
\item We have constructed the kernel in this paper in terms of trigonometric functions, whereas in \cite{cloningercluster} the kernel was constructed from Hermite polynomials. The trigonometric kernel is much faster in implementations for two reasons: 1) each individual polynomial is extremely quick to compute and 2) the trigonometric kernel deals only with trigonometric polynomials up to degree $n$, whereas the Hermite polynomial based kernel needs polynomials up to degree $n^2$ to achieve the same support estimation bounds.

\item This paper deals with arbitrary compact metric spaces (allowing for a rescaling of the data so that the maximum distance between values is $\leq \pi$), whereas \cite{cloningercluster} dealt with compact subsets of the Euclidean space and had a requirement on the degree of the kernel dependent upon the diameter of the data in terms of Euclidean distance.

\item In \cite{cloningercluster}, an algorithm known as Cautious Active Clustering (CAC) was developed. In this paper we present a new algorithm with several implementation advantages over CAC. 
We discuss this topic in more depth in Section~\ref{subsec:CACcomparison}.
\een

An important aspect of active learning is how samples are queried to minimize uncertainty. A study of two types of uncertainty in active learning problems is discussed in \cite{sharma-active}. The two critical types of uncertainty are 1) a data point is likely to belong to multiple labels, 2) a data point is not likely to belong to any label. In our algorithm, we deal with uncertain points after we have finished querying and extending the graph components. At this point, the only unlabeled points will be those that belonged to clusters with conflicting queried labels (uncertainty of the first type) or thresholded out at the beginning (uncertainty of the second type). We utilize all of the information from the iterative portion of our algorithm to then assign these points labels in a semi-supervised fashion.

One difficulty that algorithms may face in the active learning setting is the presence of highly imbalanced data (i.e. where data associated with some class labels is much more plentiful than others). In \cite{tharwat-active}, the authors discuss two concepts that an active learning algorithm should employ to be successful: exploration and exploitation. During the exploration phase, their algorithm seeks out points to query in low-sought regions. During the exploitation phase, their algorithm seeks to query points in the most critical explored regions. Our algorithm balances these principles in a different way: by querying points which we believe to be in high-density portions of a label's support (exploitation) and extending the label to nearby points until it ``bumps" against points which may belong to another label (exploration).

There are many other approaches to active learning that we categorize broadly into three groups: diffusion geometry, graph-based approaches, and neural networks. Our work has significant overlap with diffusion geometry and graph-based approaches, but no direct ties to neural networks. We list \cite{tharwat-survey} as a resourceful survey of recent developments in active learning.\\

\textbf{Graph-based algorithms:}

Graph-based active learning is a broad category of active learning algorithms which leverage graph structures for clustering. The broad approach for many of these algorithms is two-part: 1) to come up with some weighting between unlabeled data points based on the known labels to construct a weighted graph 2) minimizing or maximizing some function over the nodes of the graph to decide on the next point to query. The algorithms will switch between steps 1 and 2 iteratively to query more points and improve the classification. In \cite{calder-laplace,chen-batch,miller-uncertainty,bertozzi-model-change}, for instance, weight functions based on the graph Laplacian are applied to the data and minimized among the unlabeled data to choose successive query points. The weight function is chosen so as to decay away from the labeled data in a manner to quantify the uncertainty or information-gain among the unlabeled data points. In \cite{miller-dirichlet}, the function is chosen to be a \textit{Dirichlet variance} with the purpose of representing Bayesian information regarding plausible labels for the unlabeled data. These approaches use successive querying schemes, where each new query point requires a new calculation across all of the unlabeled data.

Our method also uses graph clustering. In contrast to many algorithms in this setting though, we intentionally look at disconnected graphs and, in particular, extend queried labels to all of the points in each of the graph components. Since our work assumes a known metric $\rho$, the distance between points by this metric works as a similarity scheme for the purpose of graph construction. Furthermore, we use the graph structure only as a clustering tool, and do not decide on points to query via any weights defined on edges. Rather, we use a density estimator to query high density points in each of the clusters we find in a hierarchical manner. \\

\textbf{Diffusion Geometry:}

Diffusion geometry active learning can be considered as a subset of graph-based active learning, where the edge weights are decided by a \textit{diffusion distance}, given in terms of a Markov transition matrix. Unlike the graph-based algorithms above which iteratively query points in a two-step process, the diffusion geometry methods mentioned here create an ordering of the data points based on the weights and do the querying based on this ordering all in one step. In Section~\ref{sec:numerical}, we have compared our algorithm with  two state-of-the-art methods: Learning by Active Nonlinear Diffusion (LAND) and Learning by Evolving Nonlinear Diffusion (LEND) algorithms \cite{murphy-land,murphy-lend}. We note that there also exist results from a diffusion geometry perspective focused on unsupervised learning \cite{murphy-unsupervised,polk-multiscale}, but both works show the potential improvements to be gained from an active learning framework.
Like the present work, diffusion geometry approaches use a kernel-based density estimation when deciding points to query. However, LAND and LEND both use a Gaussian kernel applied on $k$ neighbors for the density estimation and weight it by a diffusion value. 
Then, the queried points are simply those with the highest of the combined weights. 
The diffusion value corresponds to a minimal diffusion distance among points with a higher density estimation. 
For the point with the maximal density estimation, a maximal diffusion distance among other data points is taken as the weight. 
This extra weighting procedure is absent from our theory and algorithm, which uses an  estimation approach based purely on a localized kernel to decide on points to query. In our algorithm, we take a multiscale approach and decide on query points at each level instead of a global listing of the data points.\\

\textbf{Neural Networks:}

Neural network approaches broadly decide on points to query by inputting the data into a neural network. Although our approach is disparate from neural networks, we highlight a few recent advances in this area. In \cite{nowak-network}, an active learning approach using neural networks is developed. This work focuses on binary classification and developing models using a neural network framework such that a sufficient number of queries will achieve a desired accuracy. In \cite{schroder-network}, transformers are utilized in the active learning process, with particular emphasis on text document classification applications. In \cite{kim-network}, a method to effectively augment data sets with additional ``labeled" data points is investigated for deep active learning.

\printnomenclature

\section{Point source signal separation}\label{sec:signals}

The problem of signal separation goes back to early work of de Prony \cite{prony}, and can be stated as: estimate the coefficients $a_k$ and locations $\omega_k$ constituting $\mu=\sum_{k=1}^K a_k\delta_{\omega_k}$, from observations of the form
\be\label{eq:signal}
\hat{\mu}(x)=\hat{\mu}(x)=\sum_{k} a_k e^{-i\omega_k x}, \qquad x\in\RR.
\ee
There is much literature on methods to approach this problem, and we cite \cite{plonkafourier} as a text one can use to familiarize themselves with the topic. If we assume $\omega_k=k\Delta$ for some $\Delta\in\mathbb{R}^+$ and allow measurements for any $x\in [-\Omega,\Omega]$ for some $\Omega\in \mathbb{R}^+$, then recovery is possible so long as we are above the Rayleigh threshold, i.e. $\Omega\geq \pi/\Delta$ \cite{donoho-superres}. The case where this threshold is not satisfied is known as super-resolution. Much further research has gone on to investigate the super-resolution problem, such as \cite{batenkov-superres,candes-superres,li-superres}.

We now introduce a particular method of interest for signal separation from \cite{loctrigwave} and further developed in \cite{mhaskar-kitimoon-raj}. The method takes the following approach to estimate the coefficients and locations of $\mu$, without the assumption that the $\omega_k$'s should be at grid points, and the additional restriction that only \textbf{finitely many} integer values of $x$ are allowed. 
We start with the \textbf{trigonometric moments} of $\mu$:
$$
\hat{\mu}(\ell)=\sum_{k} a_k e^{-i\omega_k \ell}, \qquad |\ell|<n,
$$
where $n\ge 1$ is an integer.
Clearly, the quantities $\hat{\mu}(\ell)$ remain the same if any $\omega_k$ is replaced by $\omega_k$ plus an integer multiple of $2\pi$. Therefore, this problem is properly treated as the recuperation of a periodic measure $\mu$ from its Fourier coefficients rather than the recuperation of a measure defined on $\RR$ from its Fourier transform.
Accordingly, we define the quotient space $\mathbb{T}=\mathbb{R}/(2\pi\ZZ)$, and denote in this context, $|x-y|=|(x-y)\mbox{ mod } 2\pi|$.\yadi{$\TT$}{Quotient space $\RR/(2\pi\ZZ)$, sometimes referred to as the unit circle}
Here and in the rest of this paper, we consider a smooth band pass filter $h$; i.e., an even function $h\in C^\infty(\mathbb{R})$ such that $h(u)=1$ for $|u|\leq 1/2$ and $h(u)=0$ for $|u|\geq 1$. \yadi{$h$}{Smooth, bandpass filter}
We then define
\yadi{$\sigma_n$}{Reconstruction operator in different settings, \eqref{eq:sigmaintro} and \eqref{eq:sigmametricdef}}
\be\label{eq:sigmaintro}
\sigma_n(\mu)(x)\coloneqq \sum_{|\ell|<n} h\left(\frac{\ell}{n}\right)\hat{\mu}(\ell)e^{i\ell x}, \qquad x\in \mathbb{T}.
\ee
With the kernel defined by \yadi{$\Phi_n$}{Localized trigonometric polynomial kernel, \eqref{eq:trigkerndef}}
\be\label{eq:trigkerndef}
\Phi_n(t)\coloneqq \sum_{|k|<n} h\left(\frac{k}{n}\right)e^{ikt}, \qquad t\in \mathbb{T},
\ee
it is easy to deduce that
\be\label{eq:trigconv}
\sigma_n(\mu)(x)=\frac{1}{2\pi}\int_\TT \Phi_n(x-t)d\mu(t)=\sum_{k} a_k \Phi_n(x-\omega_k).
\ee
A key property of $\Phi_n$ is the \textbf{localization property} (cf. \cite{mhaskar-prestin-filbir, kitimoon2025localized}, where the notation is different): For any integer $S\ge 3$,
\be\label{eq:kernloc}
|\Phi_n(t)|\le 7\sqrt{\frac{\pi}{2}}\left\{\int_{-1}^1 |h^{(S)}(t)|dt \right\}\frac{n}{\max(1, (n|t|)^S)}.
\ee
Together with the fact that $h=1$ on $[-1/2,1/2]$, this implies that $\Phi_n$ is approximately a Dirac delta supported at $0$; in particular,
$$
\sigma_n(\mu)(x)\approx \sum_k a_k\delta_{\omega_k}(x).
$$
The theoretical details of this sentiment are described more rigorously in \cite{mhaskar-prestin-filbir, kitimoon2025localized}. 
Here, we only give two examples to illustrate.

\begin{example}\label{uda:pointsource}
{\rm
We consider the measure 
\be\label{eq:udapointsource}
\mu=5\delta_{-1}+30\delta_{2}+20\delta_{2.05},
\ee
so that the data is
\be\label{eq:udamoments}
\hat{\mu}(\ell)=5\exp(i\ell)+30\exp(-2i\ell)+20\exp(-2.05i\ell), \qquad |\ell|<n.
\ee
In Figure~\ref{fig:pointsource}, we show the graphs of the ``power spectrum'' $|\sigma_n(\mu)(x)|$ for $n=64$ and $n=256$.

\begin{figure}[h]
\begin{center}
\begin{minipage}{0.45\textwidth}
\includegraphics[width=\textwidth]{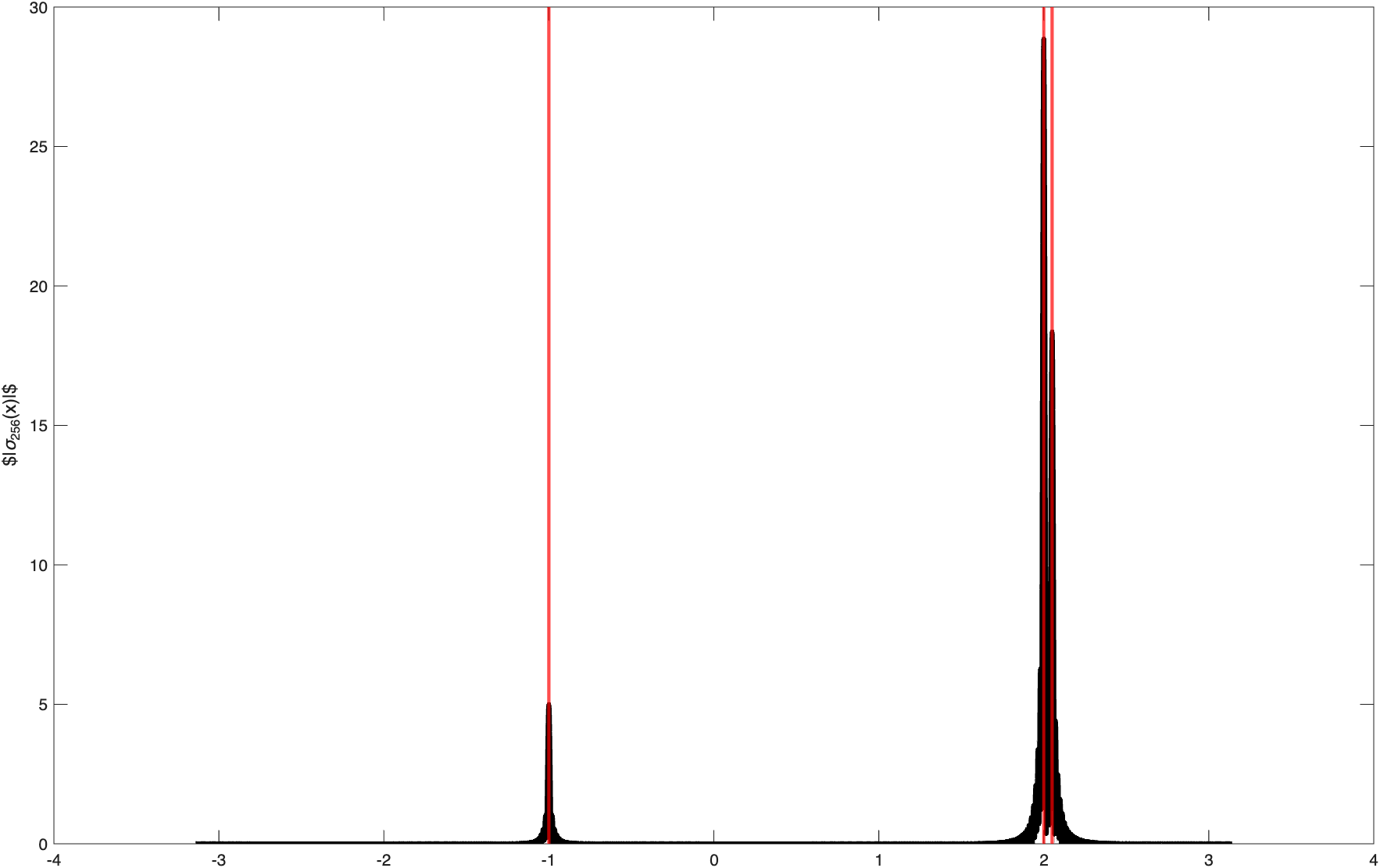} 
\end{minipage}
\begin{minipage}{0.45\textwidth}
\includegraphics[width=\textwidth]{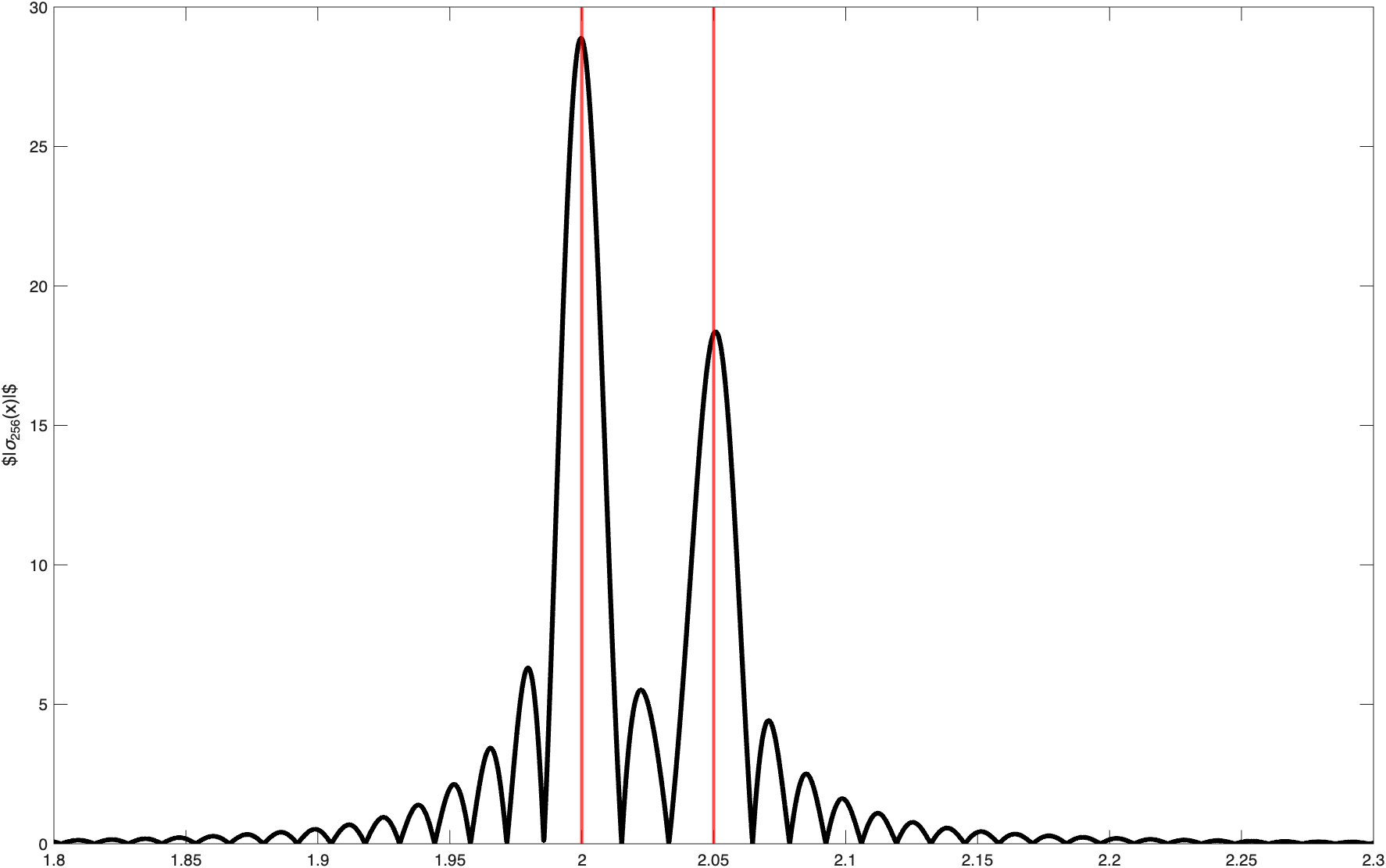} 
\end{minipage}
\begin{minipage}{0.45\textwidth}
\includegraphics[width=\textwidth]{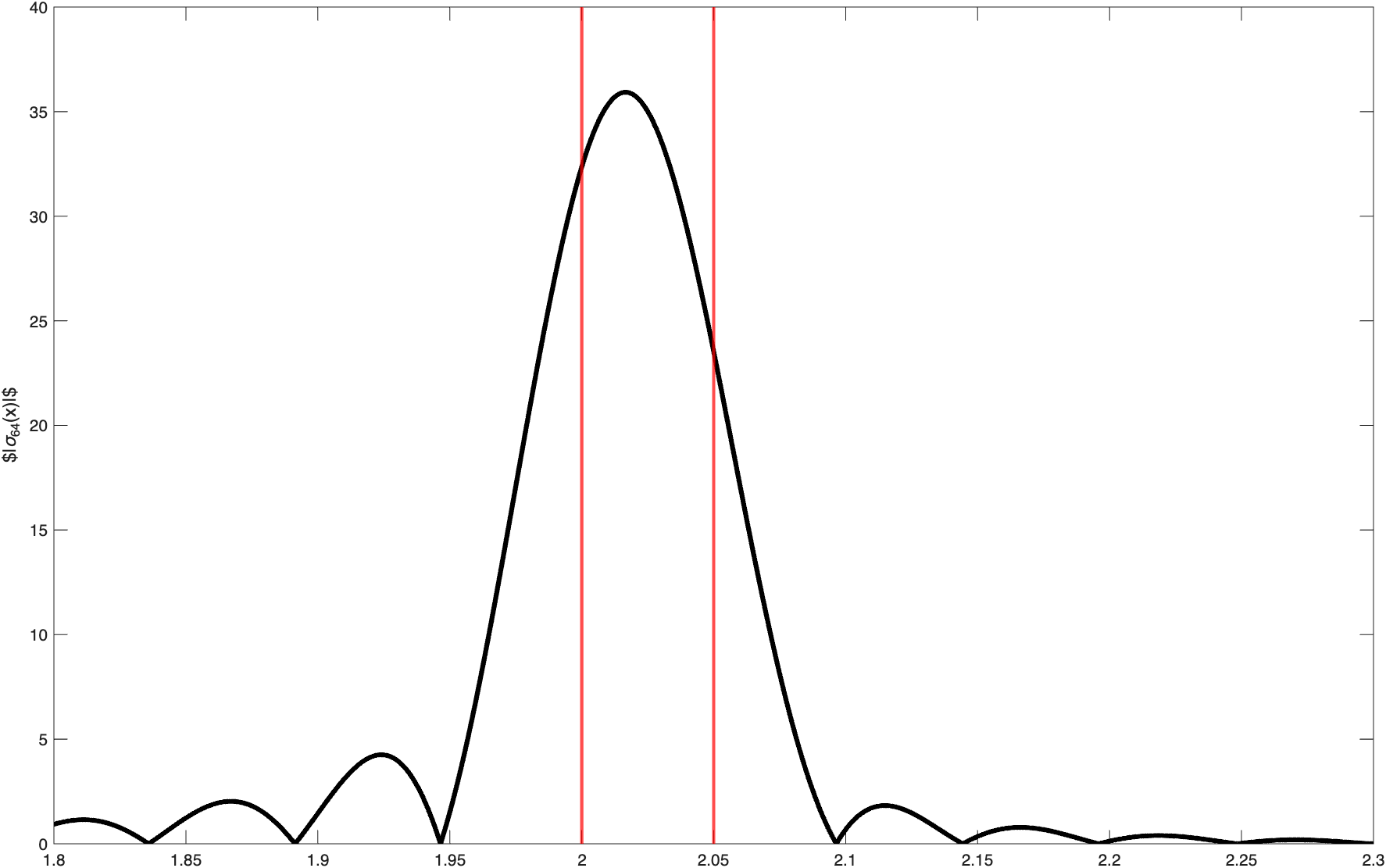} 
\end{minipage}
\end{center}
\caption{Left: $|\sigma_{256}(\mu)(x)|$ has peaks at the points $-1, 2, 2.05$, and is small everywhere else. Vertical red lines indicate the positions of these points. Right: A close-up view of $|\sigma_{256}(\mu)(x)|$ near $x=2$ to show an accurate detection of the close-by points $2, 2.05$. Bottom: A close-up view of $|\sigma_{64}(\mu)(x)|$ near $x=2$ to show the non-detection of the close-by points $2, 2.05$. }
\label{fig:pointsource}
\end{figure}

We see from the figure on the left that $\abs{\sigma_{256}}$ has peaks at approximately $-1, 2, 2.05$, and is very small everywhere else on $[-\pi,\pi]$. The figure on the right is a close-up view to highlight an accurate detection of the close-by point sources $2, 2.05$. The figure on the bottom shows that with $n=64$, such a resolution is not possible. 
When we wish to automate this, we need to figure out a threshold so that we should look only at peaks above the threshold. As the middle figure shows, there are sidelobes around each peak (and in fact, small sidelobes at many other places on $[-\pi,\pi]$).
In theory, this threshold is $\min_k |a_k|/2$, which we do not know in practice. If we set it too low, then we might  ``detect'' non-existent point sources near $2, 2.05$.
On the other hand, if we set it too high, then we would lose the low amplitude point source at $-1$. 
Some ideas on how to set an appropriate threshold, especially in the presence of noise are discussed in \cite{kitimoon2025localized}.
Another important quantity is the minimal separation $\eta=\min_{k\neq j}|(\omega_k-\omega_j) \mbox{ mod } 2\pi|$ among the point sources. 
As the middle and right figures show, the detection of point sources which are very close-by requires the knowledge of a larger number of moments.
It is shown in \cite{beyondsuper} that one must have $n\gtrsim \eta^{-1}$ in order to have sufficient resolution to recover the point sources in a stable manner.
\qed}
\end{example}

\begin{remark}\label{rem:minseparation}
{\rm
We make some remarks about the notion of minimal separation introduced in Example~\ref{uda:pointsource}. In the case of uniform sampling, the minimal separation takes on the same value as the sampling rate $\Delta$ from before. However, this perspective of minimal separation also allows one to consider the case of non-uniform sampling.
The analogue of the Rayleigh limit is a theorem in \cite{beyondsuper} showing that a stable recovery of the signal parameters require at least $\Omega(\eta^{-1})$ Fourier coefficients.
A relevant concept is that of the finite rate of innovation \cite{vetterlisignal}. 
This is defined in terms of the number of parameters involved in a reconstruction formula for the signal in the form
\be\label{eq:shannon}
\sum_{n\in\ZZ}\sum_{r=0}^{R-1} c_{r,n}\psi_n((x-t_n)/T).
\ee
The time points $t_n$ at which the signal is sampled and the coefficients $c_{r,n}$ are called the degrees of freedom in the signal. The rate of innovation is then defined as a density of these degrees of freedom in the interval over which the signal is sampled.
In our formulation, we may imagine that the signal is $\hat{\mu}$ and it is sampled at the time points $\ell$.
However, we have not taken the viewpoint that we need to use any reconstruction formula analogous to \eqref{eq:shannon}. 
Indeed, we do not even think of $\hat{\mu}$ is the signal.
We think of this as the Fourier coefficients of the signal $\mu$.
Therefore, the idea of innovation rate is not applicable in our setting.
\qed}
\end{remark}
Our next example is a precursor of the main results of this paper.

\begin{example}\label{uda:measureseparation} {\rm
We define a probability measure $\mu$ on $\TT$ as a convex combination of:
\bit
\item A sum of two uniform distributions each supported on $ [-0.6,-0.4]$, with a weight of $1200/3900$.

\item A normal distribution with mean $0.05$ and variance $0.04$, with a weight of $2400/3900$.


\item Three point-mass measures at $-2,0.4,1.5$, with weights of $60/3900, 120/3900, 120/3900$ respectively (anomaly).
\eit

We take 3900 samples from this measure (the number of points from each part of the measure corresponding to the numerator of the weight). 
The samples from the measure are visualized in Figure~\ref{fig:histogram} (a) as a normalized histogram. 
Then, we apply $\sigma_{128}$ to get an estimation of the support, as seen in Figure~\ref{fig:histogram} (b). 
Not only do we get an idea of the support of the measure by looking at $\sigma_{128}$, but also the amplitudes of the non-atomic components of the measure. 
 Since we are dealing with finitely many samples, we in fact are only estimating the integral in the definition of $\sigma_n$ as a Monte-Carlo type summation. That is, with data $\{u_j\}_{j=1}^M$ sampled randomly from $\mu$, we estimate
$$
\sigma_n(t)\approx \frac{1}{M}\sum_{j=1}^M\Phi_n(t- u_j).
$$

\begin{figure}[!ht]
\begin{center}
\begin{subfigure}[t]{.45\textwidth}
\begin{center}
\includegraphics[width=\textwidth]{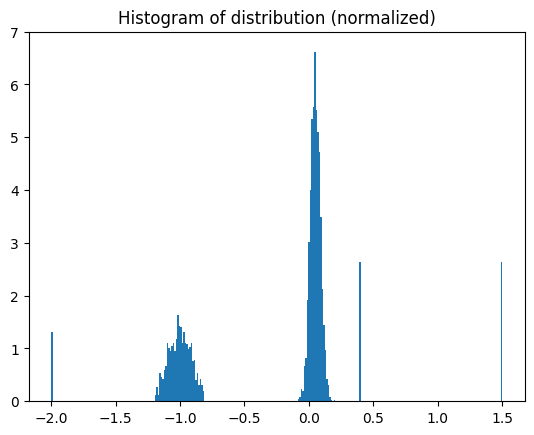}
\caption{Histogram.}
\end{center}
\end{subfigure}
\begin{subfigure}[t]{.45\textwidth}
\begin{center}
\includegraphics[width=\textwidth]{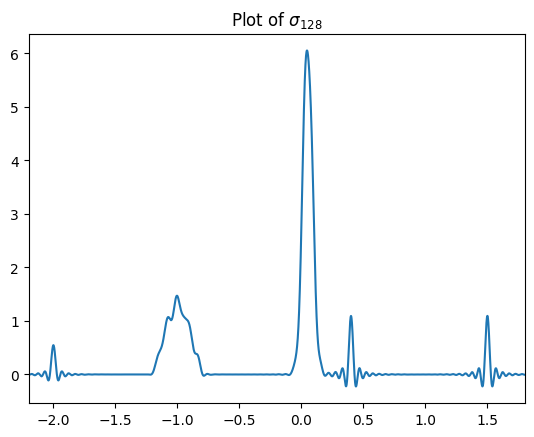}
\caption{Plot of $\sigma_{128}$.}
\end{center}
\end{subfigure}
\caption{Normalized histogram of the density of interest (left), paired with our density estimation by $\sigma_{128}$ based on $3900$ samples (right).} 
\label{fig:histogram}
\end{center}
\end{figure}
\qed}
\end{example}

In this paper, we will use similar ideas to separate the support of probability measures $\mu_k$ (rather than $\delta_{\omega_k}$), based on random samples taken from the convex combination $\mu=\sum_k a_k\mu_k$ (rather than Fourier coefficients of a general linear combination), where all the measures are supported on a compact metric measure space. 
Table~\ref{tab:supervsclass} summarizes the similarities and differences between the signal separation problem and classification problem as studied in this paper.

\begin{table}
\begin{center}
\begin{tabular}{|c|c|c|}
\hline
 & Signal separation & Classification\\
\hline
Measure: & $\mu=\sum_{k} a_k \delta_{\omega_k}$ & $\mu=\sum_{k}a_k\mu_k$\\
\hline
Domain: & $\mathbb{R}/(2\pi\ZZ)$ & unknown subset of a metric space\\
\hline
Data: & Fourier moments & samples from $\mu$\\
\hline
Key quantity: & $\min_{j\neq k}\abs{(\omega_j-\omega_k) \mod 2\pi}$ & $\min_{j\neq k}\ \operatorname{dist}(\operatorname{supp}(\mu_j),\operatorname{supp}(\mu_k))$\\
\hline
\end{tabular}
\caption{Comparison between traditional signal separation and our approach to machine learning classification.}
\label{tab:supervsclass}
\end{center}
\end{table}

\section{Background}
\label{sec:background}

In this section, we introduce many common notations and definitions used throughout the rest of this paper.

Let $\mathbb{M}$ be a compact metric space with metric $\rho$, normalized so that $\mathsf{diam}(\mathbb{M})=\max_{x,y\in\MM}\rho(x,y) =\pi$.
\yadi{$\MM$}{Ambient space}
\yadi{$\rho$}{Metric on $\MM$}
\yadi{$\XX$}{Support of the data measure $\mu$}
This normalization facilitates our use of the $2\pi$-periodic kernel $\Phi_n$, while avoiding the possibility that points $x,y$ with $\rho(x,y)\approx 2m\pi$ for some integer $m$ would be considered close to each other. 
It is well known in approximation theory that positive kernels leads to a saturation and, hence, would not be appropriate for approximating probability measures as is commonly done.
However, in this paper our main interest is to find supports of the measures rather than approximating the measures themselves.
So, in order to avoid cancellations, we prefer to deal with a positive kernel defined by
\be\label{eq:metrickerndef}
\Psi_n(x,y)=\Phi_n(\rho(x,y))^2,
\ee
where $\Phi_n$ is the kernel defined in \eqref{eq:trigkerndef}. \yadi{$\Psi_n$}{Localized kernel on $\MM$, \eqref{eq:metrickerndef}}
The localization property \eqref{eq:kernloc}, used with $\lceil S/2\rceil$ in place of $S$, implies that
\be\label{eq:metrickernloc}
\Psi_n(x,y) \le c\frac{n^2}{\max(1, (n\rho(x,y))^S)},
\ee
where $c>0$ is a constant depending only on $h$ and $S$.

At this point we would like to simplify notation by using the following constant convention.\\

\textbf{The constant convention}\\
\textit{In the sequel, $c, c_1,\cdots$ will denote generic positive constants depending upon the fixed quantities in the discussion, such as the metric space, $\rho$, and the various parameters such as $S$ and $\alpha$ (to be introduced below). 
Their values may be different at different occurrences, even within a single formula.
The notation $A\lesssim B$ means $A\le cB$, $A\gtrsim B$ means $B\lesssim A$, and $A\sim B$ means $A\lesssim B\lesssim A$.
In some cases where we believe it may be otherwise unclear, we will clarify which values a constant may depend on, which may appear in the subscript of the above-mentioned symbols. For example, $A\lesssim_d B$ means there exists $c(d)>0$ such that $A\leq c(d)B$.}\\

For any point $x\in\MM$, and any sets $A,B\subseteq \MM$, we define the following notation for the balls and neighborhoods.
\be\label{eq:balldef}
\ba
\dist(x,A)=&\inf_{y\in A} \rho(x,y),& \mathbb{B}(x,r)=&\{y\in\MM:\rho(x,y)\leq r\},\\
\dist(A,B)=&\inf_{y\in A}\dist(y,B),& \mathbb{B}(A,r)=&\{x\in\MM: \dist(x,A)\leq r\}.
\ea
\ee
\yadi{$\mathbb{B}(x,r)$}{Ball of radius $r$ centered at $x$}
\yadi{$\mathbb{B}(A,r)$}{Set of $r$-neighborhood around $A$, \eqref{eq:balldef}}
For any $A\subseteq\mathbb{M}$, we define $\operatorname{diam}(A)=\sup_{x,y\in A} \rho(x,y)$. \yadi{$\operatorname{diam}$}{Diameter of a subset of a metric space}

\subsection{Measures}
\label{subsec:measuredefs}

Let $\mu$ be a positive, Borel, probability measure on $\mathbb{M}$ (i.e. $\int_{\mathbb{M}}d\mu(y)=1$). 
We denote $\mathbb{X}\coloneqq \mathsf{supp}(\mu)$.
Much of this paper focuses on $\XX$. 
However, we wish to treat $\XX$ as an \textbf{unknown} subset of a known ambient space $\MM$ rather than treating it as a metric space in its own right. 
In particular, this emphasizes the fact the data measure $\mu$  may not have a density, and may not be supported on the entire ambient space.

In the case of signal separation, we have seen that if the minimal amplitude for a certain point source is sufficiently small, we may not be able to detect that point source. 
Likewise, if the measure $\mu$ is too small on parts of $\XX$, we may not be able to detect those parts.
For this reason, we make some assumptions on the measure $\mu$ as in  \cite{cloningercluster}. 
The first property, \textit{detectability}, determines the rate of growth of the measure locally around each point in the support. 
The second property, \textit{fine-structure}, relates the measure to the classification problem by equipping the support with some well-separated (except maybe for some subset of relatively small measure) partition which may correspond to some different class labels in the data. 

\begin{definition}
\label{def:detectable}
We say a measure $\mu$ on $\mathbb{M}$ is \textbf{detectable} if there exist $\alpha\geq 0, \kappa_1, \kappa_2>0$ such that 
\be\label{eq:ballmeasurecon}
\mu(\mathbb{B}(x,r))\le \kappa_1 r^\alpha,\qquad x\in \mathbb{M},\ r>0,
\ee
and there exists $r_0>0$ such that
\be\label{eq:ballmeasurelow}
\mu(\mathbb{B}(x,r))\ge \kappa_2 r^\alpha,\qquad x\in\mathbb{X}, \ 0<r\le r_0.
\ee
\end{definition}

\begin{definition}
\label{def:finestructure}
We say a measure $\mu$ has a \textbf{fine structure} if there exists an $\eta_0$ such that for every $\eta\in (0,\eta_0]$ there is an integer $K_\eta$ and a partition $\mathbf{S}_\eta\coloneqq\{\mathbf{S}_{k,\eta}\}_{k=1}^{K_\eta+1}$ of $\mathbb{X}$ where both of the following are satisfied.
\ben
\item (\textbf{Cluster Minimal Separation}) For any $j,k=1,2,\dots,K_\eta$ with $j\neq k$ we have
\be\label{eq:minsep}
\operatorname{dist}(\mathbf{S}_{j,\eta},\mathbf{S}_{k,\eta})\geq 2\eta.
\ee
\yadi{$\eta$}{Minimal separation among classes at different levels}
\yadi{$K_\eta$}{Number of partition elements of $\XX$ separated by $2\eta$, Definition~\ref{def:finestructure}}
\item (\textbf{Exhaustion Condition}) We have
\be\label{eq:exhaustion}
\lim_{\eta\to 0^+} \mu(\mathbf{S}_{K_\eta+1,\eta})=0.
\ee
\een
We will say that $\mu$ has a \textbf{fine structure in the classical sense} if $\mu=\sum_{k=1}^K a_k\mu_k$ for some probability measures $\mu_k$, $a_k$'s are $>0$ and $\sum_k a_k=1$, and the compact subsets $\mathbf{S}_k\coloneqq \mathsf{supp}(\mu_k)$ are disjoint. 
In this case $\eta$ is the minimal separation among the supports and there is no overlap.
 
\end{definition}

\begin{remark}\label{rem:exception}
{\rm
It is possible to require the condition \eqref{eq:ballmeasurelow} on a subset of $\XX$ having measure converging to $0$ with $r$. 
This will add some difficulties in our proof of \eqref{eq:Jn} and Lemma~\ref{lemma:discprelim}. However, in the case when $\mu$ has a fine structure, this exceptional set can be absorbed in $\mathbf{S}_{K_\eta+1}$ with appropriate assumptions. 
We do not find it worthwhile to explore this further in this paper.
\qed}
\end{remark}

\begin{example}
{\rm
Supposing that $\mu=\sum_{k=1}^K a_k \delta_{\omega_k}$ as in the signal separation problem, then we see that $\mu$ is detectable with $\alpha=0$, $\kappa_1=\sum_k |a_k|$, $\kappa_2=\min_k |a_k|$. In this context, the value $\kappa_2$ is known as the minimum weight  and plays an important role in signal recovery (recall the threshold $\min_{k}\abs{a_k}/2$ from Example~\ref{uda:pointsource}). The measure $\mu$ also has fine structure in the classical sense whenever $\eta<\min_{j\neq k} |\omega_j-\omega_k|$. In this sense, the theory presented in this paper is a generalization of results for signal separation in this regime.
\qed} 
\end{example}

\begin{example}
{\rm
If $\mathbb{X}$ is a $\alpha$-dimensional, compact, connected, Riemannian manifold, then the normalized Riemannian volume measure is detectable with parameter $\alpha$. 
\qed}
\end{example}

\subsection{F-score}\label{sec:fscore}

We will give results on the theoretical performance of our measure estimating procedure by giving an asymptotic result involving the so-called F-score. 
The F-score for binary classification (true/false) problems is a measure of classification accuracy taking the form of the harmonic mean between precision and recall. 
In a predictive model, precision is defined as the fraction of true positive outputs over all the positive outputs of the model. 
Recall is the fraction of true positive outputs over all the actual positives. 
In a multi-class problem, we extend this definition as follows (cf. \cite{diagraphcluster_ohiostate_2011}). 
If $\{C_1,\dots,C_N\}$ is a partition of $\{x_j\}_{j=1}^M$ indicating the predicted output labels of a model and $\{L_1,\dots,L_K\}$ is the ground-truth partition of the data, then 
one can define the precision of $C_j$ against the true label $L_k$ by $|C_j\cap L_k|/|C_j|$ and the corresponding recall by $|C_j\cap L_k|/|L_k|$.
Taking the maximum of the harmonic means of the precisions and recalls with respect to all the ground truth labels leads to 
\be
F(C_j)=2\max_{k\in \{1,\dots,K\}} \frac{\abs{C_j\cap L_k}}{\abs{C_j}+\abs{L_k}}.
\ee
Then the F-score is given by
\be
F\left(\{C_j\}_{j=1}^N\right)=\frac{\sum_{j=1}^N \abs{C_j}F(C_j)}{\sum_{j=1}^N \abs{C_j}}.
\ee
Since we are treating the data as samples from a measure $\mu$, we replace cardinality in the above formulas with measure. Our fine structure condition gives us the true supports as $\{\mathbf{S}_{k,\eta}\}_{k=1}^{K_\eta}$ for any valid $\eta$, so we can define the F-score for the support estimation clusters $\{\mathcal{C}_{j,\eta}\}_{j=1}^N$ by
\be\label{eq:class_fscore}
\mathcal{F}_\eta(\mathcal{C}_{j,\eta})=2\max_{k\in \{1,\dots,K\}}\frac{\mu(\mathcal{C}_{j,\eta}\cap \mathbf{S}_{k,\eta})}{\mu(\mathcal{C}_{j,\eta})+\mu(\mathbf{S}_{k,\eta})},
\ee
and \yadi{$\mathcal{F}_\eta$}{$F$-score at separation $\eta$, \eqref{eq:fscoredef}}
\be\label{eq:fscoredef}
\mathcal{F}_\eta\left(\{\mathcal{C}_{j,\eta}\}_{j=1}^N\right)=\frac{\sum_{j=1}^N \mu(\mathcal{C}_{j,\eta})\mathcal{F}_\eta(\mathcal{C}_{j,\eta})}{\mu\left(\bigcup_{j=1}^N \mathcal{C}_{j,\eta}\right)}.
\ee
\begin{remark}\label{rem:fscore}
We observe that 
$$
1-2\frac{\mu(\mathcal{C}_{j,\eta}\cap \mathbf{S}_{k,\eta})}{\mu(\mathcal{C}_{j,\eta})+\mu(\mathbf{S}_{k,\eta})}=\frac{\mu(\mathcal{C}_{j,\eta}\Delta \mathbf{S}_{k,\eta})}{\mu(\mathcal{C}_{j,\eta})+\mu(\mathbf{S}_{k,\eta})},
$$
where in this remark only, $\Delta$ denotes the symmetric difference.
It follows that $0\le \mathcal{F}_\eta\le 1$.
If we estimate each support perfectly so $C_{j,\eta}=\mathbf{S}_{j,\eta}$ for all $j$ and each $C_{j,\eta}$ is $\eta$-separated from any other, then we see that $\mathcal{F}_\eta\left(\{\mathcal{C}_{j,\eta}\}_{j=1}^N\right)=1$. Otherwise, we will attain an F-score strictly lower than $1$. \qed
\end{remark}

\section{Main results}
\label{sec:mainresults}

In this section we introduce the main theorems of this paper, which involve the recovery of supports of a measure from finitely many samples. Theorem~\ref{thm:fullsuportdet}
 pertains to the case where we only assume the detectability of the measure. Theorem~\ref{thm:class_separation} pertains to the case where we additionally assume the fine structure condition. Before stating the results, we must introduce our discrete measure support estimator and support estimation sets. We define our  \textbf{data-based measure support estimator} by \yadi{$F_n$}{Measure support estimator, \eqref{eq:support_estimator_def}}
\be\label{eq:support_estimator_def}
F_n(x)\coloneqq \frac{1}{M}\sum_{j=1}^M \Psi_n(x,x_j).
\ee
This definition is then used directly in the construction of our \textbf{data-based support estimation sets}, given by \yadi{$\mathcal{G}_n$}{Support estimation set, \eqref{eq:support_est_set_def}}
\be\label{eq:support_est_set_def}
\mathcal{G}_n(\Theta)\coloneqq \left\{x\in \mathbb{M}: F_n(x)\geq \Theta \max_{1\leq k\leq M} F_n(x_k)\right\}.
\ee
Intuitively, due to the localization \eqref{eq:metrickernloc} of $\Psi$, we expect $F_n$ to be large when nearby any of the data points $x_j$, and small otherwise. With this understanding, $\mathcal{G}_n$ then thresholds points in the metric space where $F_n$ is sufficiently large. This gives an estimation for the support of the data. The thresholding value $\Theta$ decides the cutoff based on the maximal value that $F_n$ attains over the data. If $\Theta$ is too large we would expect to get an underestimation of the support set $\mathbb{X}$, whereas if it is too small we will get an overestimation. We formalize this intuition in our first theorem.

\begin{remark}\label{rem:kde}
{\rm
The function $F_n$ resembles a kernel density estimator, and it would be if the data measure $\mu$ was assumed to be absolutely continuous with respect to some base measure on the ambient metric space $\mathbb{M}$.
However, we do not assume a base measure on the ambient metric space.
Moreover, even if the metric space were the Euclidean space or a sphere as we have used in Section~\ref{sec:numerical}, we do not require $\mu$ to have a density with respect to the natural base measure on these spaces; indeed, we are interested in the cases when $\mu$ is a singular measure. 
Finally, it is not our goal to approximate the measure itself, but merely to approximate its support.
\qed}
\end{remark}

\begin{theorem}\label{thm:fullsuportdet}
Let $\mu$ be detectable and suppose $M\gtrsim n^{\alpha}\log(n)$. Let $\{x_1,x_2,\dots,x_M\}$ be independent samples from $\mu$. There exists a constant $C>0$ such that if $\Theta<C<1$, then there exists $r(\Theta)\sim \Theta^{-1/(S-\alpha)}$ (recall $S$ is the localization parameter of the kernel, given in \eqref{eq:kernloc}~and~\eqref{eq:metrickernloc}) such that with probability at least $1-c_1/M^{c_2}$ we have
\be\label{eq:thm1}
\mathbb{X}\subseteq \mathcal{G}_n(\Theta)\subseteq \mathbb{B}\left(\mathbb{X},r(\Theta)/n\right).
\ee
\end{theorem}

Our second theorem additionally assumes the fine-structure condition on the measure, and gives conditions so that for any satisfactory $\eta$, the support estimation set $\mathcal{G}_n(\Theta)$ splits into $K_\eta$ subsets each with separation $\eta$, thus solving the machine learning classification problem in theory.

\begin{theorem}\label{thm:class_separation}
Suppose, in addition to the assumptions of Theorem~\ref{thm:fullsuportdet}, that $\mu$ has a fine structure, $n\gtrsim 1/(\eta\Theta^{1/(S-\alpha)})$, and $\mu(\mathbf{S}_{K_{\eta}+1,\eta})\lesssim \Theta n^{-\alpha}$. Define
\be\label{eq:thm2-1}
\mathcal{G}_{k,\eta,n}(\Theta)\coloneqq \mathcal{G}_n(\Theta)\cap \mathbb{B}(\mathbf{S}_{k,\eta},r(\Theta)/n).
\ee
Then, with probability at least $1-c_1/M^{c_2}$, $\{\mathcal{G}_{k,\eta,n}(\Theta)\}_{k=1}^{K_\eta}$ is a partition of $\mathcal{G}_{n}(\Theta)$ such that 
\be\label{eq:thm2-2}
\operatorname{dist}(\mathcal{G}_{j,\eta,n}(\Theta),\mathcal{G}_{k,\eta,n}(\Theta))\geq \eta\qquad j\neq k,
\ee 
and in this case, there exists $c<1$ such that
\be\label{eq:thm2-3}
\mathbb{X}\cap \mathbb{B}(\mathbf{S}_{k,\eta},cr(\Theta)/n)\subseteq \mathcal{G}_{k,\eta,n}(\Theta)\subseteq \mathbb{B}\left(\mathbf{S}_{k,\eta},r(\Theta)/n\right).
\ee
\end{theorem}

\begin{remark}\label{rem:meshnorm}
If $\mathcal{C}=\{z_1,\cdots, z_M\}$ is a random sample from $\mu$, $n\ge 1$ and $M\gs n^\alpha\log n$, then it can be shown (cf. \cite[Lemma~7.1]{mhaskar2020kernel}) that for any point $x\in\XX$, there exists some $z\in\mathcal{C}$ such that $\rho(x,z)\le 1/n$. 
Hence, the Hausdorff distance between $\XX$ and $\mathcal{C}$ is $\le 1/n$. 
If $\mu$ has a fine structure in the classical sense, and $n\gs \eta^{-1}$, then this implies that a correct clustering of $\mathcal{C}$ would give rise to a correct classification of every point in $\XX$.
This justifies our decision to construct the algorithm in Section~\ref{sec:alg} to classify only the points in $\mathcal{C}$. 
On the other hand, the use of the localized kernel as in the theorems above guide us about the choice of the points at which to query the label. \qed
\end{remark}

In Figure~\ref{fig:twomoons} we illustrate Theorem~\ref{thm:class_separation} applied to a simple two-moons data set. We see that the support estimation set, shown in yellow, covers the data points as well as their nearby area, predicting the support of the distribution from which the data came from. Furthermore, we show in the figure a motivating idea: by querying a single point in each component for its class label we can extend the label to the other points in order to classify the whole data set. This is how we utilize the active learning paradigm in our algorithm discussed in Section~\ref{sec:alg}.

\begin{figure}[!ht]
\begin{center}
\includegraphics[width=.4\textwidth]{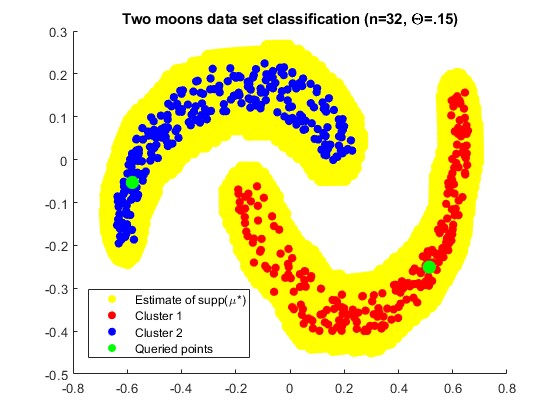}
\end{center}
\caption{Demonstration of the support estimation set $\mathcal{G}_{32}(0.15)$ (yellow) applied to a simple two-moons data set from \cite{twomoons} (blue and red). By querying one point from each component of the support estimation set and extending the label to the other points in the same component, we can classify the entire data set with 100\% accuracy.}
\label{fig:twomoons}
\end{figure}

Our final result examines the fidelity of our classification scheme in terms of the asymptotics of the F-score associated with our support estimation theorems as $\eta\to 0$.
 We show that our support estimation setup asymptotically approaches the ideal F-score of $1$.
   
\begin{theorem}\label{thm:fscore}
Suppose the assumptions of Theorem~\ref{thm:class_separation} are satisfied and that
\be
\lim_{\eta\to 0^+} \max_{0\leq k\leq K_\eta}\left(\frac{\mu(\mathbf{S}_{K_{\eta+1},\eta})}{\mu(\mathbf{S}_{k,\eta})}\right)=0.
\ee
Then, with probability at least $1-c_1/M^{c_2}$, we have
\be
\lim_{\eta\to 0^+} \mathcal{F}_\eta\left(\{\mathcal{G}_{k,\eta,n}(\Theta)\}_{k=1}^{K_\eta}\right)=1.
\ee
where $\mathcal{F}_\eta$ is the $F$-score with respect to $\mathbf{S}_\eta$.
\end{theorem}

\section{MASC algorithm}
\label{sec:alg}

\subsection{Algorithm description}

In the following paragraphs we describe the motivation and intuition of the algorithm MASC (Algorithm~\ref{alg:MASC}). Throughout this section we will refer to line numbers associated with Algorithm~\ref{alg:MASC}.

One obvious way to embed a data into a metric space with diameter $\le \pi$ is just to rescale it.
If the data is a compact subset of an ambient Euclidean space $\RR^q$, we may project the data on the unit sphere $\SS^q\subset \RR^{q+1}$ by a suitable inverse stereographic projection. 
The metric space $\SS^q$, equipped with the geodesic distance $\arccos(\circ,\circ)$ has diameter $\pi$ by construction. In any case, we assume that we have access to $\rho(x_i,x_j)$ for all $x_i,x_j\in \mathcal{D}$.

One of the main obstacles we must overcome in an implementation of our theory is the following. 
In practice, we often do not know the minimal separation $\eta$ of the data classes beforehand, nor do we know optimal values for $\Theta,n$.
Taking a machine learning perspective, we develop a multiscale approach to remedy these technical challenges: treat $n,\Theta$ as hyperparameters of the model and increment $\eta$. 
Firstly, MASC will threshold out any data points not belonging to $\mathcal{G}_n(\Theta)$ (line 2).
For each value of $\eta$ (initialize while loop in line 4) we construct a (unweighted) graph where an edge goes between two points $x_i,x_j$ if and only if $\rho(x_i, x_j)<\eta$ (line 5). 
At this point, we have a method for unsupervised clustering by simply examining graph components (line 6, see below for discussion on $p$). 
The idea to implement active learning is to then query a modal point of each graph component (line 11), also referred to in this section as a cluster, with respect to $\Psi_n$ and extend that label to the rest of the cluster (line 13). 
A trade-off associated with this idea is the following: if we initialize $\eta$ too small (respectively, $n$ too large) then each point in the data set will be its own cluster and we will simply query the whole data set, whereas if we initialize $\eta$ too large (respectively, $n$ too small) then the whole data set will belong to a single cluster destroying any classification accuracy. 
Therefore, we initialize $\eta$ small and introduce a minimum cluster size threshold value $p$ to avoid this issue. Any cluster of size $<p$ will be removed from consideration (line 6), so we will not query any points until $\eta$ is large enough to produce a cluster of size $p$ or greater.

After the label extension is done in each cluster of size $\ge p$, we keep track of which points we queried (line 12), increment $\eta$ (line 16), and repeat (line 4). 
Sometime after the first incrementation of $\eta$, we will experience the combination of clusters which were previously disconnected. 
When this occurs we check whether each of the previously queried points in the new cluster have the same label (line 14). If so, then we extend it to the new cluster (line 15).
 Otherwise, we halt the extension of labels for all points in that cluster. 
 In this way, the method proceeds by a cautious clustering to avoid labeling points that are either 1) in a too-low density region, or 2) within a cluster where we have queried multiple points with contradicting labels.

Once $\eta$ is large enough that the data set all belongs to a single cluster, we will not gain any new information by incrementing $\eta$ further, and hence MASC will halt the iterations of $\eta$ (lines 7 and 8). The final process is to implement a method for estimating the labels of points that did not receive a predicted label in the first part, either because they belonged to a low-density region and were thresholded out or because they belonged to a cluster with conflicting queried points. The remaining task is equivalent to the semi-supervised regime of classification and we acknowledge that there is a vast variety of semi-supervised learning methods to choose from. In MASC, we have elected to use a traditional $\overline{k}$-nearest neighbors approach.

For a data point $x_j$, we denote the set of its nearest $\overline{k}$ neighbors which already have labels $\hat{y}(x_j)$ estimated from MASC by $\mathcal{A}_{j,\overline{k}}$. The $\overline{k}$-nearest neighbors formula to estimate the label of $x_j$ is then given by:
\be
\underset{k\in [K]}{\argmax} |\{x_i\in A_{j,\overline{k}}: \hat{y}(x_j)=k\}|,
\ee
with some way to decide on the choice of $k$ in the event of a tie. In binary classification tasks, the value of $\overline{k}$ can be chosen as an odd value to prevent ties. Otherwise, a tie can be broken by choosing the label of the nearest point with a tied label, a hierarchical ordering of the labels, at random, etc. In our Python implementation of the algorithm used to produce the figures in this paper, we use the \verb|scipy.stats.mode| function, which returns the first label in the list of tied labels upon such a tie.

MASC will collect all points which do not yet have predicted labels (line 17), and apply the nearest-neighbors approach as described above to each of these points (lines 19 and 20). At this point, every element in the data set will have a predicted label, so the algorithm will return the list of labels (line 21).


In MASC, we require defining a starting $\eta$ and $\eta_{\operatorname{step}}$. Once the matrix with entries given by $\Psi_n(x_i,x_j)$ is calculated, one may search for the range of $\eta$ values which give non-trivial clusters of size $\geq p$ with relative ease. If $\eta$ is too small, no cluster will contain a sufficient number of points and if $\eta$ is too large, every point will belong to the same cluster, both of which we consider a ``trivial" case. Then $\eta_{\operatorname{step}}$ may be chosen to satisfy some total number of iterations across this domain. The values $n,\Theta,p,\overline{k}$ are considered hyperparameters.

\begin{algorithm}
\caption{Multiscale Active Super-resolution Classification (MASC)}
\label{alg:MASC}
\KwIn{Data set $X$, kernel degree $n$, threshold parameter $\Theta$, $\eta$ initialization, step size $\eta_{\text{step}}>0$, cluster size minimum $p$, oracle $f$, neighbor parameter $\overline{k}$.}
\KwOut{Predicted labels $\hat{y}$ for all points in $X$.}

$\mathcal{A} \leftarrow \emptyset$ \Comment*{Initialize queried point set}
    $V\gets \{x_i\in X: x_i\in \mathcal{G}_n(\Theta)\}$ \Comment*{Prune data to consider only those in threshold set \eqref{eq:support_est_set_def}}
$\operatorname{STOP}\gets \operatorname{FALSE}$ \;

\While{$\operatorname{STOP}=\operatorname{FALSE}$}{
	$E\gets \{(x_i,x_j)\in V\times V:\rho(x_i,x_j)<\eta, x_i\neq x_j\}$ \Comment*{Edge set consisting of points within $\eta$ distance from each other}
    $\{C_{\eta,\ell}\}_{\ell=1}^{K_\eta}\gets$ connected components of $G=(V,E)$ with size $\ge p$ \;
     \If{$|C_{\eta,1}|=|V|$}{
     $\operatorname{STOP}\gets\operatorname{TRUE}$ \Comment*{End while loop once $G$ is connected}}

    \For{$\ell = 1$ \KwTo $K_n$}{
        \If{$C_{\eta,\ell} \cap \mathcal{A} = \emptyset$}{
            $x_i \leftarrow \underset{{x \in C_{\eta,\ell}}}{\argmax} \sum_{j=1}^{M} \Psi_n( x, x_j)$ \Comment*{Locate maximizer of $F_n$ (cf. \eqref{eq:support_estimator_def}) in $C_{\eta,\ell}$ without any queried points}
            $\mathcal{A} \leftarrow \mathcal{A}\bigcup \{x_i\}$ \Comment*{Append maximizer to queried point set}
            {\small $\hat{y}(x_j) \leftarrow f(x_i)$ for all $x_j \in C_{\eta,\ell}$} \Comment*{Query point and extend label to all of $C_{\eta,\ell}$}
        }
        \ElseIf{$\forall x_i,x_j \in C_{\eta,\ell} \cap \mathcal{A}, f(x_i) = f(x_j)=:c_{\eta,\ell}$}{
            {\small $\hat{y}(x_j) \leftarrow c_{\eta,\ell}$ for all $x_j \in C_{\eta,\ell}$ \Comment*{If all queried points in component have same label, extend label to entire component}} 
        }
    }

        $\eta \leftarrow \eta + \eta_{\text{step}}$ \;

}

%

%

$\mathcal{C}_{\text{uncertain}}\gets \{x\in X : \hat{y}(x_j)=\text{DNE}\}$ \Comment*{Set of points which do not have a predicted label}

\For{$x_j \in \mathcal{C}_{\text{uncertain}}$}{
$\mathcal{A}_{j,\overline{k}}\gets \{x\in X\setminus \mathcal{C}_{\text{uncertain}}: x \text{ is the $\overline{k}$th closest element to $x_j$ or closer, with respect to $\rho$}\}$ \;
    
    $\hat{y}(x_j)\gets \underset{k\in [K]}{\argmax} |\{x_i\in A_{j,\overline{k}}: y_j=k\}|$ \Comment*{$\overline{k}$-nearest neighbors approach to estimate labels for uncertain points}
}

\Return $\hat{y}$.
\end{algorithm}

\subsection{Comparison with CAC and SCALe}
\label{subsec:CACcomparison}

In \cite{cloningercluster}, a similar theoretical approach to this paper except on the Euclidean space was developed and an algorithm we will call ``Cautious Active Clustering" (CAC) was introduced. MASC and CAC are both multiscale algorithms using $\mathcal{G}_n(\Theta)$ to threshold the data set, then constructing graphs to query points and extend labels. The main difference between the algorithms is the following. In CAC, $\eta,\Theta$ are considered hyperparameters while $n$ is incremented, whereas in MASC, $n,\Theta$ are considered hyperparameters while $\eta$ is incremented. 
This adjustment serves three purposes:
\ben
\item It connects the algorithm closer to the theory, which states that a single $n,\Theta$ value will suffice for the right value of $\eta$. We do not know $\eta$ in advance, but by incrementing $\eta$ until all of the data belongs to a single cluster, we will attain a value close to the true value at some step. At this step, we will query points belonging roughly to the ``true" clusters and that information will be carried onward to the subsequent steps.

\item Consistency in query procedure: we use the same function to decide which points to query at each level, rather than it changing as the algorithm progresses.

\item It improves computation times since computing the $\Psi_{n}$ matrix for varying values of $n$ tends to take more time than incrementing $\eta$ and checking graph components.
\een
In MASC, we have the additional parameter $p$ specifying the minimum size of the graph component to allow a query. While this is new compared to CAC, the main purpose is to reduce the total number of queries to just those that contain more information. One could implement such a change to CAC as well for similar effect. A further difference is that CAC uses a localized summability kernel approach to classify uncertain samples, whereas MASC uses a nearest-neighbors approach.

SCALe, as introduced in \cite{mhaskar-odowd-tsoukanis} is an even more similar algorithm to MASC. The main difference between MASC and SCALe is the final step, where in the present method we use a nearest-neighbors approach to extend labels to uncertain points while in SCALe the choice was to use a  function approximation technique developed in \cite{mhaskarodowd}. Both methods have their pros and cons. Compared to SCALe, the nearest-neighbors approach of MASC:
\ben
\item  works in arbitrary metric spaces, without requiring a summability kernel as in SCALe.

\item extends labels to uncertain points (sometimes much) faster, reducing computation time while usually providing comparable or better results with sufficiently many queries, but

\item reduces accuracy in extremely sparse query setting, where the function estimation method with the manifold assumption empirically seems to extend labels more consistently.
\een

\section{Numerical examples}
\label{sec:numerical}

In this section, we look at the performance of the MASC algorithm applied to 1) a synthetic data set with overlapping class supports (Section~\ref{subsec:circleellipse}), 2) a document data set (Section~\ref{sec:document}), and 3) two different hyperspectral imaging data sets: Salinas (Section~\ref{sec:salinas}) and Indian Pines (Section~\ref{sec:indianpines}). In each case, we project the (potentially pre-processed) data to the sphere and use $\rho=\arccos(\circ\cdot \circ)$ as the metric for graph construction. This guarantees that the metric space has diameter $\leq\pi$. On the Hyperspectral data sets, we compare our method with two other algorithms for active learning: LAND and LEND (Section~\ref{sec:comparison}). 

For hyperparameter selection on our model as well as the comparisons, we have not done any validation but rather optimized the hyperparameters for each model on the data itself. So the results should be interpreted as being near-best-possible for the models applied to the data sets in question rather than a demonstration of generalization capabilities. While this approach is non-traditional for unsupervised/supervised learning, it has been done for other active learning research (\cite{murphy-lend}, for example) so we have elected to follow the same procedure in this paper. Further, an exhaustive grid search was not conducted but rather local minima among grid values were selected for each hyperparameter. For MASC, we looked at $n$ in powers of $2$ and $\overline{k}$ values in multiples of $5$. For LAND we looked at $K,t$ at increments of $10$, and with LEND we used the same parameters from LAND and looked at integer $J$ values and $\alpha$ values in increments of $0.1$. For $\Theta$, we tried values less rigorously, meaning that better $\Theta$ values may exist than the ones chosen. Due to the nature of the algorithm, increasing $\Theta$ will increase the number of samples that the nearest-neighbors approach has to estimate, while reducing the number of labeled neighbors it has to do so. However, increasing $\Theta$ can also reduce the number of queries used, sometimes without deterioration in accuracy. So there may be some tradeoff, but we generally see the best results when $\Theta$ is chosen to threshold a small portion of the initial data (outlier removal). In Table~\ref{tab:parameters}, we summarize the choice of parameters for each of the data sets in the subsequent sections.

\begin{table}[htbp]
    \centering
    \textbf{MASC hyperparameter selection for each data set}\\
    \begin{tabular}{|l|c|c|c|c|}
        \hline
        Dataset & $\Theta$ & $\eta$ & $p$ & $\overline{k}$ \\
        \hline
        Circle+Ellipse & 0.12 & $[0.006, 0.036]$ & 15 & 5 \\
        (Section~\ref{subsec:circleellipse})               &      & (step size 0.005) &    &   \\
        \hline
        Document & 0.51 & $[0.08, 0.15]$ & 3 & 25 \\
        (Section~\ref{sec:document})         &      & (step size 0.002) &    &   \\
        \hline
        Salinas & 0.32 & $[0.21, 0.27]$ & 3 & 25 \\
        (Section~\ref{sec:salinas})       &      & (step size 0.005) &    &   \\
        \hline
        Indian Pines & 0.08 & $[0.03, 0.13]$ & 5 & 15 \\
        (Section~\ref{sec:indianpines})             &      & (step size 0.005) &    &   \\
        \hline
    \end{tabular}
        \caption{Selected hyperparameter values for our MASC algorithm applied to the data sets in the subsequent sections.}
    \label{tab:parameters}
\end{table}

\subsection{Circle on ellipse data}
\label{subsec:circleellipse}

Although the theory in this paper focuses on the case where the supports of the classes are separated (or at least satisfy a fine-structure condition), our MASC algorithm still performs well at classification tasks of data with overlapping supports in the regions without overlap. 
To illustrate this, we generated a synthetic data set of 1000 points sampled along the arclength of a circle and another 1000 sampled along the arclength of an ellipse with eccentricity 0.79. 
For each data point, normal noise with standard deviation 0.05 was additively applied independently to both components. Figure~\ref{fig:circleellipse} shows the true class label for each of the points on the left and the estimated class labels on the right. 
We can see that the misclassifications are mostly localized to the area where the supports of the two measures overlap.
 Near the intersection points of the circle and ellipse the classification problem becomes extremely difficult due to a high probability that a data point could have been sampled from either the circle or ellipse.

\begin{figure}
\begin{center}
\begin{subfigure}[t]{.45\textwidth}
\begin{center}
\includegraphics[width=\textwidth]{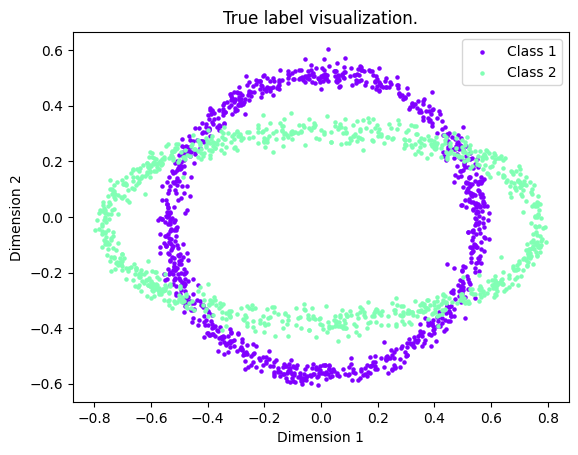}
\subcaption{True labels of the circle and ellipse data.}
\end{center}
\end{subfigure}
\begin{subfigure}[t]{.45\textwidth}
\begin{center}
\includegraphics[width=\textwidth]{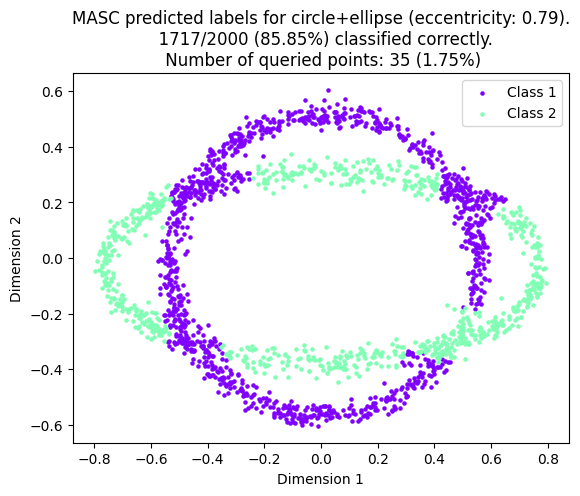}
\subcaption{Predicted labels using MASC with 35 queries, achieving 83\% accuracy.}
\end{center}
\end{subfigure}
\caption{This figure illustrates the result of applying MASC to a synthetic circle and ellipse data set. On the left are true labels of the given data, and on the right is the estimation attained by MASC.}
\label{fig:circleellipse}
\end{center}
\end{figure}

\subsection{Document data}
\label{sec:document}

This numerical example uses the document data set provided by Jensen Baxter through Kaggle \cite{documentdata}. The data set contains 1000 documents total, 100 each belonging to a particular category from: business, entertainment, food, graphics, historical, medical, politics, space, sport, and technology. For prepossessing we run the data through the Python sklearn package's TfidfVectorizer function to convert the documents into vectors of length 1684. Then we implement MASC.

In Figure~\ref{fig:doc1} we see the results of applying MASC on the document data in two steps. On the left we see the classification task by MASC paused at line 17 of Algorithm~\ref{alg:MASC}, before labels have been extended via the nearest neighbor  portion at the end of the algorithm. On the right we see the result of the density estimation extension. In Figure~\ref{fig:doc2} we see on the left a confusion matrix for the result shown in Figure~\ref{fig:doc1}, allowing us to see which classes were classified the most accurately versus which ones had more trouble. We see the largest misclassifications had to do with documents that were truly ``entertainment" but got classified as either ``sport" or ``technology", and documents which were actually ``graphics" but got classified as ``medical". On the right of Figure~\ref{fig:doc2} we have a plot indicating the resulting accuracy vs. the number of queries which MASC was allowed to do. Naturally as the number of queries approaches 1000 this plot will gradually increase to 100\% accuracy. Lastly, in Figure~\ref{fig:doc3} we see a side-by-side comparison of the true labels for the document data set vs. the predicted labels.

\begin{figure}
\begin{center}
\begin{subfigure}[t]{.45\textwidth}
\begin{center}
\includegraphics[width=\textwidth]{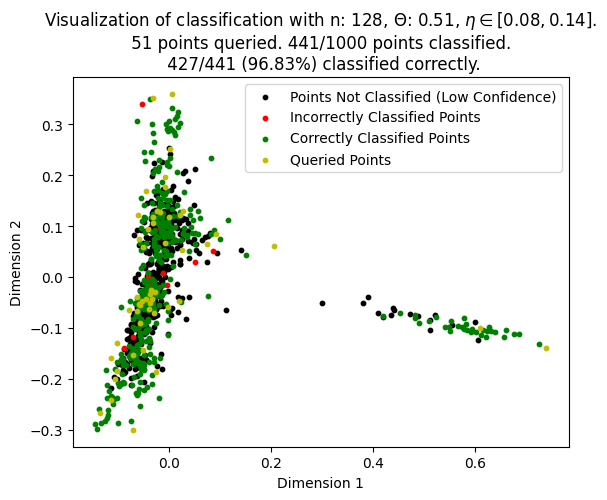}
\subcaption{Classification of certain points in MASC algorithm (before density estimation extension).}
\end{center}
\end{subfigure}
\begin{subfigure}[t]{.45\textwidth}
\begin{center}
\includegraphics[width=\textwidth]{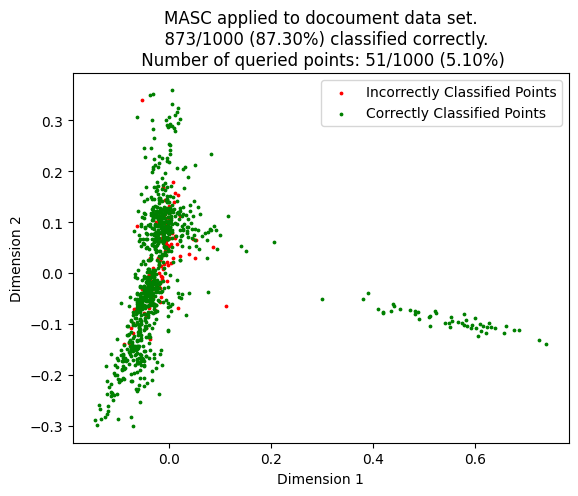}
\subcaption{Classification of remainder points using density estimation extension.}
\end{center}
\end{subfigure}
\caption{This figure illustrates the classification process undergone MASC on the document data set at two points. On the left, we see the classification of points before the $\overline{k}$-nearest neighbors extension. On the right, we see the result after $\overline{k}$-nearest neighbors extension. Figure dimensions are the directions associated with the largest 2 singular values of the data matrix.}
\label{fig:doc1}
\end{center}
\end{figure}

\begin{figure}[!ht]
\begin{center}
\begin{subfigure}[t]{.45\textwidth}
\begin{center}
\includegraphics[width=\textwidth]{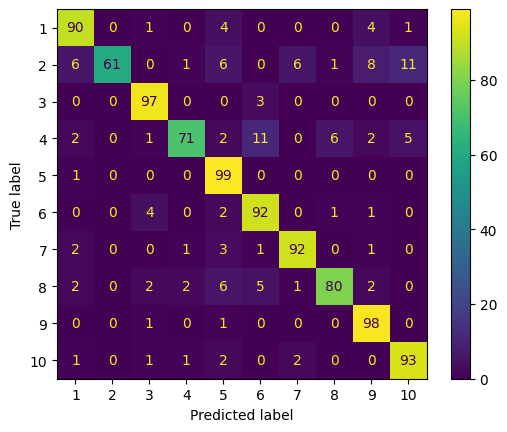}
\subcaption{Confusion matrix.}
\end{center}
\end{subfigure}
\begin{subfigure}[t]{.45\textwidth}
\begin{center}
\includegraphics[width=\textwidth]{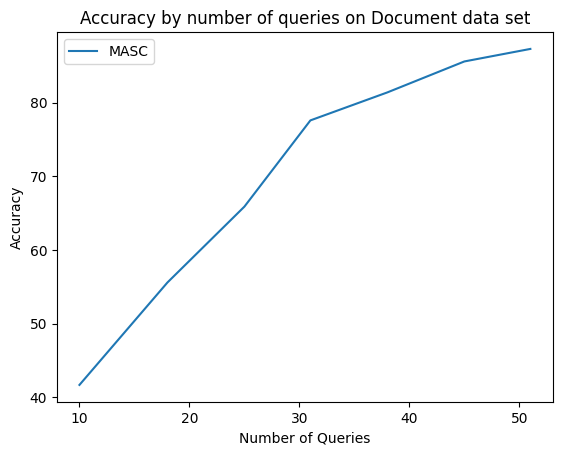}
\subcaption{Plot of MASC accuracy vs. number of allowed query points}
\end{center}
\end{subfigure}
\caption{Further details on the classification results for the document data set. (Left) Confusion matrix for single run of MASC algorithm. (Right) Accuracy of MASC algorithm vs. the number of queries used.}
\label{fig:doc2}
\end{center}
\end{figure}

\begin{figure}[!ht]
\begin{center}
\begin{subfigure}[t]{.45\textwidth}
\begin{center}
\includegraphics[width=\textwidth]{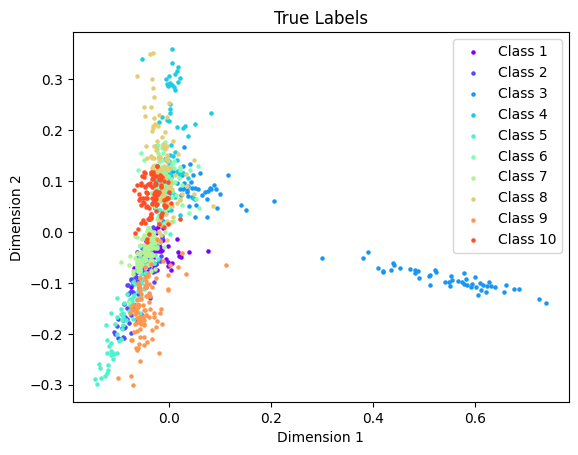}
\caption{True labels.}
\end{center}
\end{subfigure}
\begin{subfigure}[t]{.45\textwidth}
\begin{center}
\includegraphics[width=\textwidth]{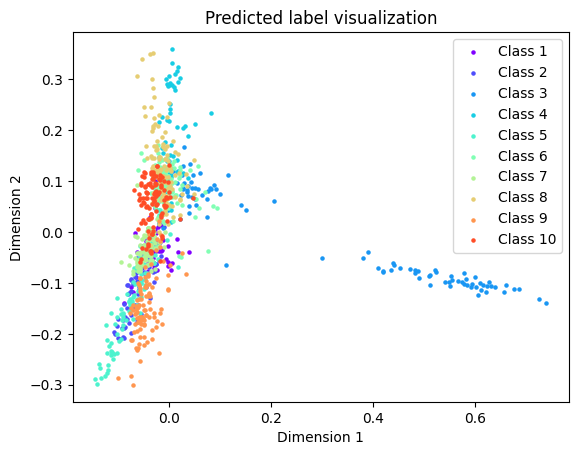}
\caption{Predicted labels.}
\end{center}
\end{subfigure}
\caption{Visual comparison of true labels (left) versus predicted labels output by the model (right) for the document data set.}
\label{fig:doc3}
\end{center}
\end{figure}

\subsection{Salinas hyperspectral data}
\label{sec:salinas}

This numerical example is done on a subset of the Salinas hyperspectral image data set from \cite{hsidata}. The full data set is visualized in Figure~\ref{fig:sal_gt}. Our subset of the Salinas data set consists of 20034 data vectors of length 204 belonging to 10 classes of the 16 original classes. Specifically, we took half of the data points at random from each of the first 10 classes of the original data set. For preprocessing we ran PCA and kept the first 50 components. Then we implemented MASC.

\begin{figure}[!ht]
\begin{center}
\includegraphics[width=.4\textwidth]{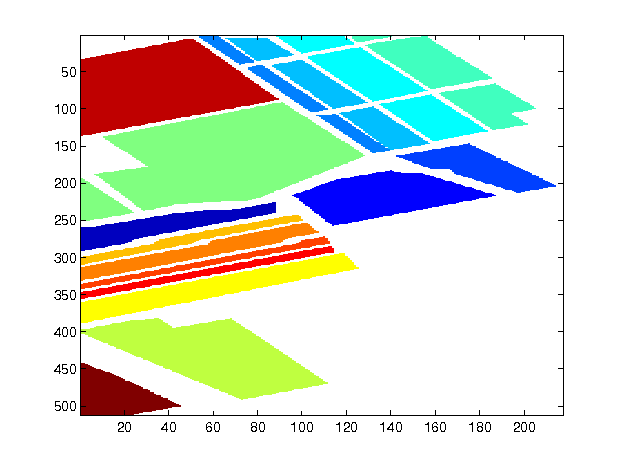}
\caption{Visualization of the full Salinas data set ground truth labels by geographic location. Image sourced from \cite{hsidata}.}
\label{fig:sal_gt}
\end{center}
\end{figure}

In Figure~\ref{fig:sal1} we see the results of applying MASC on the Salinas data in two steps. 
On the left we see the classification task by MASC paused at line 17 of Algorithm~\ref{alg:MASC}, before labels have been extended via the nearest neighbor  portion at the end of the algorithm. 
At this stage, our algorithm has classified 1518 points with 99.60\% accuracy using 261 queries. 
On the right we see the result of the $\overline{k}$-nearest neighbors extension, where all 20034 points have been classified with 97.11\% accuracy. 
In Figure~\ref{fig:sal2} we see a confusion matrix for the result shown in Figure~\ref{fig:sal1}, allowing us to see which classes were classified the most accurately versus which ones had more trouble. We see the largest misclassification involved our predicted class 5, which included points from several other classes. Lastly, in Figure~\ref{fig:sal3} we see a side-by-side comparison of the true labels for the Salinas data set versus the predicted labels.

\begin{figure}[!ht]
\begin{center}
\begin{subfigure}[t]{.45\textwidth}
\begin{center}
\includegraphics[width=\textwidth]{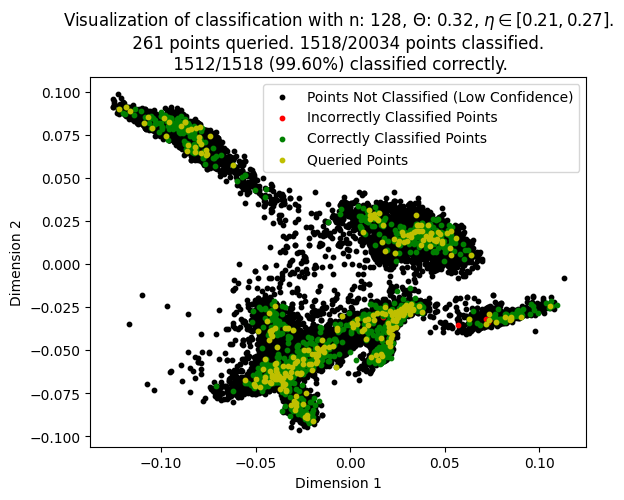}
\subcaption{Classification of certain points in MASC algorithm (before $\overline{k}$-nearest neighbors extension).}
\end{center}
\end{subfigure}
\begin{subfigure}[t]{.45\textwidth}
\begin{center}
\includegraphics[width=\textwidth]{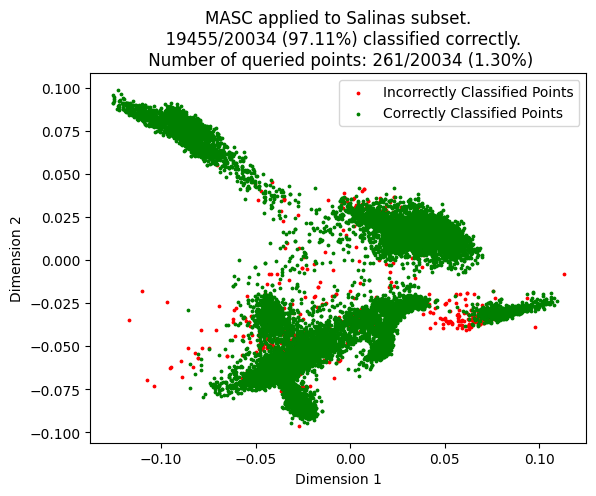}
\subcaption{Classification of remainder points using $\overline{k}$-nearest neighbors extension.}
\end{center}
\end{subfigure}
\caption{This figure illustrates the classification process undergone by MASC at two points on the Salinas hyperspectral data set. On the left, we see the classification of points before the $\overline{k}$-nearest neighbors extension. On the right, we see the result after $\overline{k}$-nearest neighbors extension. Dimensions correspond to the second and third directions of greatest variance according to the PCA decomposition.}
\label{fig:sal1}
\end{center}
\end{figure}

\begin{figure}[!ht]
\begin{center}
\includegraphics[width=.45\textwidth]{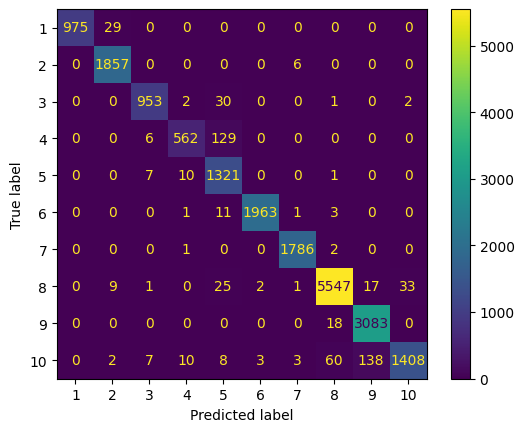}
\caption{Confusion matrix for single run of MASC algorithm on Salinas.}
\label{fig:sal2}
\end{center}
\end{figure}

\begin{figure}[!ht]
\begin{center}
\begin{subfigure}[t]{.45\textwidth}
\begin{center}
\includegraphics[width=\textwidth]{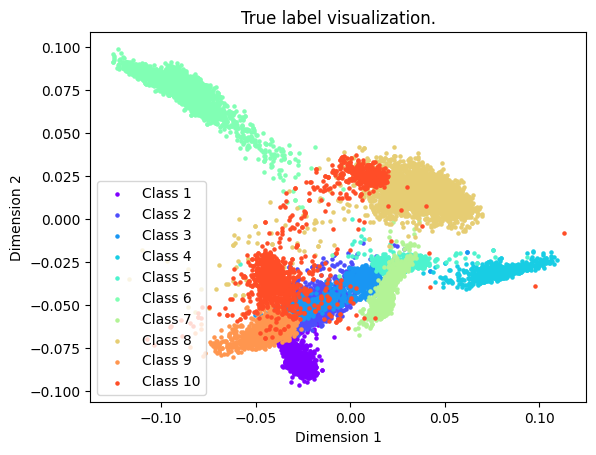}
\caption{True labels.}
\end{center}
\end{subfigure}
\begin{subfigure}[t]{.45\textwidth}
\begin{center}
\includegraphics[width=\textwidth]{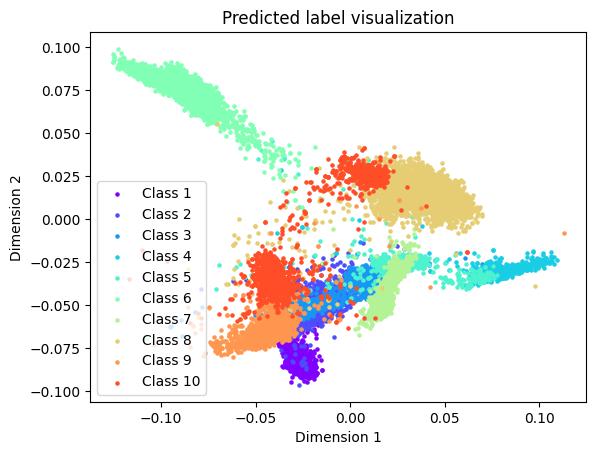}
\caption{Predicted labels.}
\end{center}
\end{subfigure}
\caption{Visual comparison of true labels (left) versus predicted labels output by the model (right) for the Salinas hyperspectral data set.}
\label{fig:sal3}
\end{center}
\end{figure}

\subsection{Indian Pines hyperspectral data}
\label{sec:indianpines}

This numerical example is done on a 5-class subset of the Indian Pines hyperspectral image data set from \cite{hsidata}. The full data set is visualized in Figure~\ref{fig:ip_gt}. Our subset of the Indian Pines data set consists of 5971 data vectors of length 200 belonging to classes number 2,6,11,14,16 of the 16 original classes. For preprocessing we normalized each vector. Then we implement MASC.

\begin{figure}[!ht]
\begin{center}
\includegraphics[width=.4\textwidth]{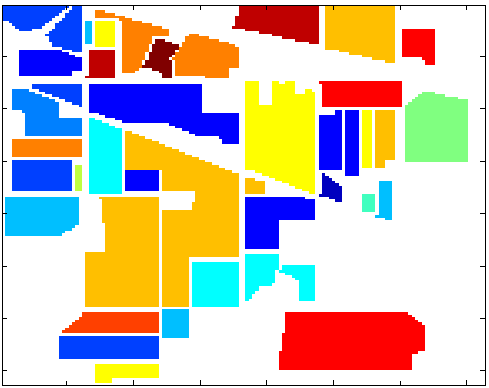}
\caption{Visualization of the full Indian Pines data set ground truth labels by geographic location. Image sourced from \cite{hsidata}.}
\label{fig:ip_gt}
\end{center}
\end{figure}

In Figure~\ref{fig:ip1} we see the results of applying MASC on the Indian Pines data in two steps. On the left we see the classification task by MASC paused at line 17 of Algorithm~\ref{alg:MASC}, before labels have been extended via the nearest neighbor portion at the end of the algorithm. On the right we see the result of the $\overline{k}$-nearest neighbors extension. In Figure~\ref{fig:ip2} we see a confusion matrix for the result shown in Figure~\ref{fig:ip1}, allowing us to see which classes were classified the most accurately versus which ones had more trouble. As we can see from the confusion matrix, the largest error comes from distinguishing class 2 from 11 and vise versa. These classes correspond to portions of the images belonging to corn-notill and soybean-mintill. Lastly, in Figure~\ref{fig:ip3} we see a side-by-side comparison of the true labels for the Indian Pines data set versus the predicted labels.

Lastly in Figure~\ref{fig:ip4}, we show how points are assigned estimated labels over several sequential iterations of the algorithm. This figure demonstrates the multiscale approach of the algorithm, wherein high density points tend to be assigned labels early on in the iterations and low density points tend to be assigned labels later. Black points in the figure correspond to those not yet given labels after the shown iterations. Some of the black points correspond to very low density points which may have been thresholded out at the beginning of the algorithm process, while others correspond to points of potential label conflict to be resolved after the iterations.

\begin{figure}[!ht]
\begin{center}
\begin{subfigure}[t]{.45\textwidth}
\begin{center}
\includegraphics[width=\textwidth]{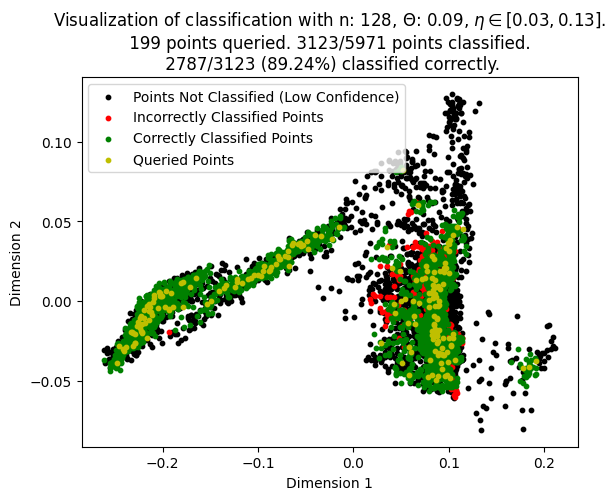}
\subcaption{Classification of certain points in MASC algorithm (before $\overline{k}$-nearest neighbors extension).}
\end{center}
\end{subfigure}
\begin{subfigure}[t]{.45\textwidth}
\begin{center}
\includegraphics[width=\textwidth]{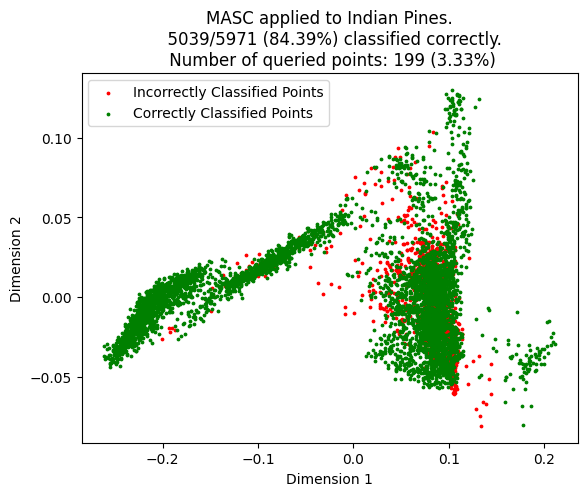}
\subcaption{Classification of remainder points using $\overline{k}$-nearest neighbors extension.}
\end{center}
\end{subfigure}
\caption{This figure illustrates the classification process undergone by MASC at two points on the Salinas hyperspectral data set. On the left, we see the classification of points before the $\overline{k}$-nearest neighbors extension. On the right, we see the result after $\overline{k}$-nearest neighbors extension. Figure dimensions are the directions associated with the largest 2 singular values of the data matrix.}
\label{fig:ip1}
\end{center}
\end{figure}

\begin{figure}[!ht]
\begin{center}
\includegraphics[width=.45\textwidth]{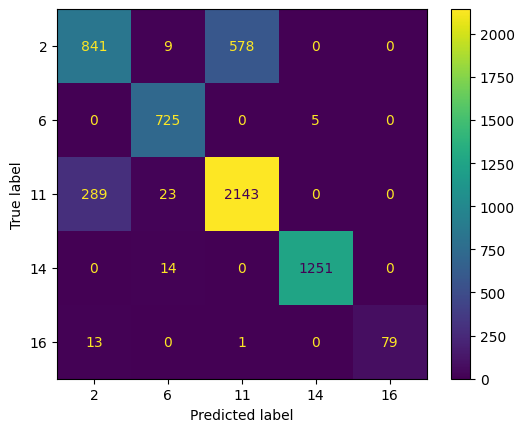}
\caption{Confusion matrix for result of MASC applied to Indian Pines.}
\label{fig:ip2}
\end{center}
\end{figure}

\begin{figure}[!ht]
\begin{center}
\begin{subfigure}[t]{.45\textwidth}
\begin{center}
\includegraphics[width=\textwidth]{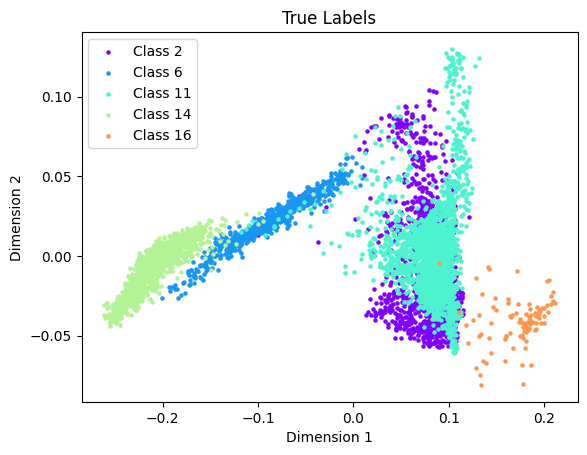}
\caption{True labels.}
\end{center}
\end{subfigure}
\begin{subfigure}[t]{.45\textwidth}
\begin{center}
\includegraphics[width=\textwidth]{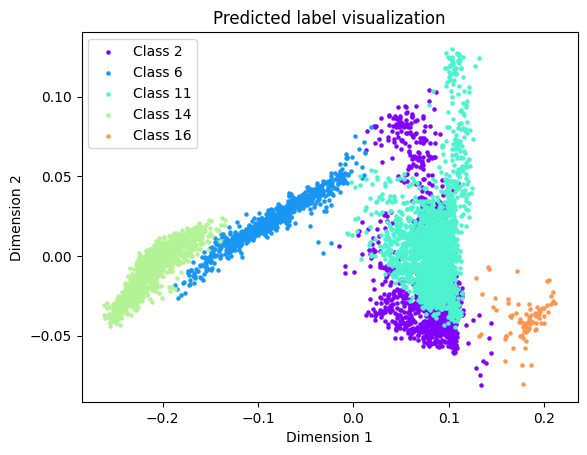}
\caption{Predicted labels.}
\end{center}
\end{subfigure}
\caption{Visual comparison of true labels (left) versus predicted labels output by the model (right) for the Indian Pines hyperspectral data set.}
\label{fig:ip3}
\end{center}
\end{figure}

\begin{figure}[!ht]
\begin{center}
\includegraphics[width=.6\textwidth]{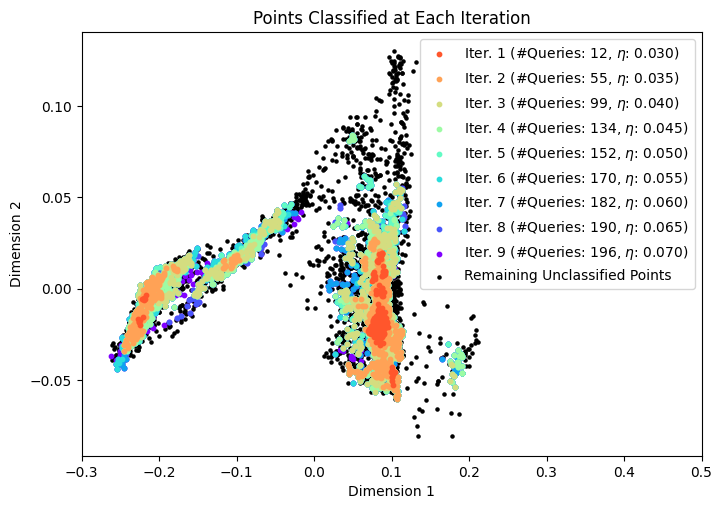}
\caption{Visualization of label assignment by iteration. Shown are the points given new estimated labels by the time MASC has reached the iteration with the shown $\eta$ value. Also depicted in the legend of the plot is the number of queries used by said iteration. The iterations are shown sequentially ($\eta$ increasing) but not consecutively (intermediate iterations not shown and new points from those iterations lumped into the shown iterations).}
\label{fig:ip4}
\end{center}
\end{figure}

\subsection{Comparison with LAND and LEND}
\label{sec:comparison}

We compare our method with the LAND \cite{murphy-land} algorithm and its boosted variant, LEND \cite{murphy-lend}. 
In Figure~\ref{fig:comparison}, we see the resulting accuracy that each algorithm achieves on both Salinas and Indian Pines for various query budgets. 
On the left, we observe that our method achieves a comparable accuracy to both LAND and LEND at around 50 queries, then gradually surpasses the accuracy of LAND as the number of queries surpasses around 200. 
On the right, our method achieves a lower accuracy for a small number of queries, but then outperforms both LAND and LEND after the budget exceeds about 60 queries.

The query budgets were decided by how many queries were used at various $\eta$ levels of while loop in the MASC Algorithm~\ref{alg:MASC}. We then forced the nearest-neighbors portion of the MASC algorithm to extend labels to the remainder of the data set at each such level, which is shown in the plot.

A separate aspect of comparison involves the run-time of both algorithms. In Table~\ref{tab:salcompare}, we see that while LEND has the highest accuracy on the Salinas data set with 261 queries, it takes significantly longer than the other two methods to attain this result. Of the three methods, MASC has the quickest run-time at 110.8s, achieving a better accuracy than LAND in less time. In Table~\ref{tab:ipcompare}, we see that MASC produces both the best result and has the fastest run-time for the case of 211 queries on the Indian Pines data set.

When deciding which algorithm to use for an active learning classification task, one has to consider the trade off between query budget/cost, computation time, and accuracy. Our initial results indicate that if the query cost is not so high compared to the run-time of the algorithm, then one may elect to use MASC with its lower run-time and simply query more points. However, if the query cost is high compared to the run-time, then one may instead elect to use an algorithm like LEND instead. The comparison results in this section are not meant to give an exhaustive depiction of which algorithm to use in any case, only illustrate that in two data sets of interest, MASC performs competitively with the existing methods in terms of either or both accuracy and run-time.

\begin{figure}[!ht]
\begin{center}
\begin{subfigure}[t]{.45\textwidth}
\begin{center}
\includegraphics[width=\textwidth]{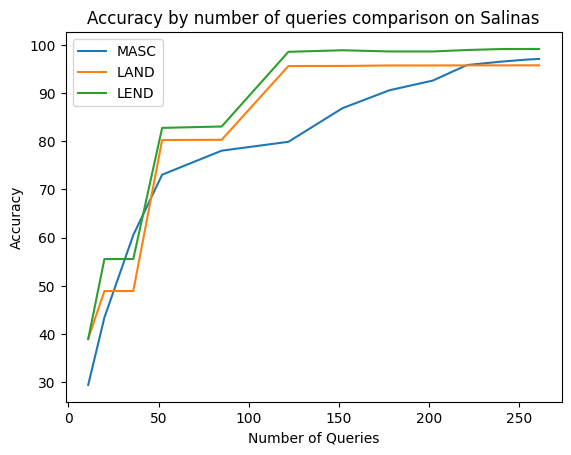}
\subcaption{Plot of accuracy vs. number of query points for Salinas.}
\end{center}
\end{subfigure}
\begin{subfigure}[t]{.45\textwidth}
\begin{center}
\includegraphics[width=\textwidth]{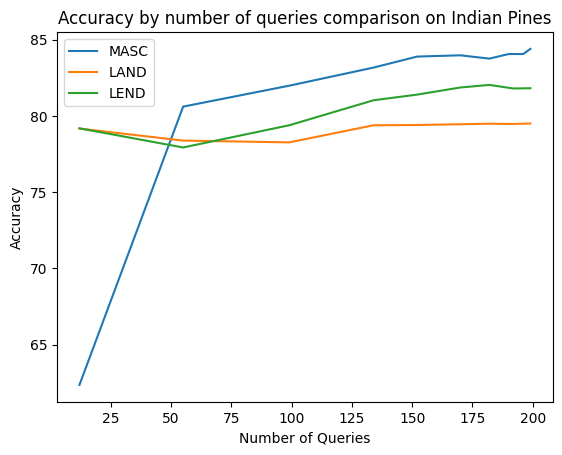}
\subcaption{Plot of accuracy vs. number of query points for Indian Pines.}
\end{center}
\end{subfigure}
\caption{Plots indicating the accuracy of MASC, LAND, and LEND for different query budgets, for both Salinas (left) and Indian Pines (right).}
\label{fig:comparison}
\end{center}
\end{figure}

\begin{table}
\begin{center}
\textbf{Comparison of MASC with LAND and LEND on Salinas subset}
\begin{tabular}{|c|c|c|c|}
\hline
Salinas & MASC & LAND & LEND\\
\hline
Accuracy & 97.1\% & 95.7\% & \textbf{99.2\%}\\
\hline
Run-time & \textbf{110.8s} & 190.0s & 669.1s\\
\hline
\end{tabular}
\end{center}
\caption{Comparison between MASC, LAND, and LEND on the Salinas data set using 261 queries.}
\label{tab:salcompare}
\end{table}

\begin{table}
\begin{center}
\textbf{Comparison of MASC with LAND and LEND on Indian Pines subset}
\begin{tabular}{|c|c|c|c|}
\hline
 & MASC & LAND & LEND\\
\hline
Accuracy & \textbf{84.4\%} & 79.5\% & 82.8\%\\
\hline
Run-time & \textbf{15.5s} & 19.6s & 97.6s\\
\hline
\end{tabular}
\end{center}
\caption{Comparison between MASC, LAND, and LEND on the Indian Pines data set using 211 queries.}
\label{tab:ipcompare}
\end{table}

\section{Proofs}
\label{sec:proofs}

In this section we give proofs for our main results in Section~\ref{sec:mainresults}. We assume that $\mathbb{X}\coloneqq \operatorname{supp}(\mu)\subseteq \mathbb{M}$ and $n\geq 1$ is given. Essential to our theory is the construction of an integral support estimator:
\be\label{eq:sigmametricdef}
\sigma_n(x)\coloneqq \int_{\mathbb{X}}\Psi_n(x,y)d\mu(y).
\ee
We also define the following two associated values which will be important:
\be\label{eq:InJndef}
I_n\coloneqq \max_{x\in\MM}|\sigma_n(x)|,\qquad J_n\coloneqq \min_{x\in\mathbb{X}}\abs{\sigma_n(x)}.
\ee
Informally, we expect the evaluation of $\sigma_n(x)/I_n$ to give us an estimation on whether or not the point $x$ belongs to $\mathbb{X}$. We encode this intuition by setting a thresholding (hyper)parameter $\theta>0$ in a support estimation set:
\be\label{eq:Sn}
\mathcal{S}_n(\theta)\coloneqq\left\{x\in\mathbb{M}:\sigma_n(x)\geq 4\theta I_n\right\}.
\ee
When the measure $\mu$ is detectable, we show that $\mathcal{S}_n(\theta)$ is an estimate to the support of $\mu$ (Theorem~\ref{thm:suppmu}). When the measure $\mu$ has a fine structure, we show that $\mathcal{S}_n(\theta)$ is partitioned exactly into $K_\eta$ separated components and each component estimates the support of the corresponding partition $\mathbf{S}_{k,\eta}$ (Theorem~\ref{thm:partmu}). These results then give us give us the ability to estimate the classification ability in the discrete setting via probabilistic results, as we investigate in Section~\ref{sec:discreteproofs}.

\subsection{Measure support estimation}
\label{sec:measureproofs}

In this section we develop key results to estimate the supports of measures defined on a continuum.
We first start with a useful lemma giving upper and lower bounds on $I_n,J_n$ respectively. Additionally for any given $x\in\mathbb{M}$, we determine a bound for the integral of $\Psi_n$ taken over points away from $x$.

\begin{lemma}\label{lem:integralest}
Let $n\geq 1$ and $S>\alpha$. Then there exist $C_1,C_2>0$ (depending on $\alpha,S,h$) such that
\be\label{eq:In}
I_n=\max_{x\in\MM}|\sigma_n(x)| \leq C_1n^{2-\alpha}
\ee
and
\be\label{eq:Jn}
J_n=\min_{x\in\mathbb{X}}\abs{\sigma_n(x)}\geq C_2n^{2-\alpha}.
\ee
\yadi{$C_1$}{Constant in the upper bound for $\sigma_n$, \eqref{eq:In}}
\yadi{$C_2$}{Constant in a lower bound for $\sigma_n$, \eqref{eq:Jn}}
In particular, $C_1\geq C_2$. For $d>0$ and any $x\in\mathbb{M}$,
\be\label{eq:Kn}
\int_{\mathbb{M}\setminus \mathbb{B}(x,d)}\Psi_n(x,y)d\mu(y)\leq C_1\frac{n^{2-\alpha}}{\max(1,(nd)^{S-\alpha})}.
\ee
\end{lemma}

In order to prove this lemma, we first recall a consequence of the Bernstein inequality for trigonometric polynomials (\cite{paibk}, Chapter~III, Section~3, Theorems~1 and Lemma~5).
\begin{lemma}\label{lemma:bernstein}
Let $T$ be a trigonometric polynomial of order $<2n$. 
Then
\be\label{eq:trigbern}
\|T\|=\max_{x\in\TT}|T'(x)|\le 2n\max_{x\in\TT}|T(x)|.
\ee
Moreover, if $|T(x_0)|=\|T\|$ then
\be\label{eq:triglowbd}
|T(x)|\ge \|T\|\cos(2nx), \qquad |(x-x_0) \mod 2\pi|\le \pi/(2n).
\ee
\end{lemma}
The following corollary gives a consequence of this lemma for the kernel $\Psi_n$, which will be used often in this paper.
\begin{corollary}\label{cor:kernelbern}
Let $x, y, z, w\in \MM$, $n\ge 1$. Then
there are constants $c, C_0$ such that
\be\label{eq:kernmax}
cn^2\le \Psi_n(x,x)\le C_0 n^2.
\ee
\yadi{$C_0$}{Upper bound for $\Psi_n(x,x)/n^2$, \eqref{eq:kernmax}}
Moreover,
\be\label{eq:kernupbd}
\Psi_n(x,y)\le \Psi_n(x,x)\sim n^2,
\ee
\be\label{eq:kernbern}
|\Psi_n(x,y)-\Psi_n(z,w)|\ls n^3\left\{\rho(x,z)+\rho(y,w)\right\}.
\ee
and
\be\label{eq:kernlowbd}
|\Psi_n(x,y)| \gs n^2, \qquad \mbox{ for } \rho(x,y)\le \pi/(6n).
\ee
\end{corollary}
\begin{proof} \ 
The estimate \eqref{eq:kernmax} follows from the fact that
$$
\Psi_n(x,x)=\Phi_n(0)^2 =\left(\sum_\ell h(\ell/n)\right)^2 \sim n^2,
$$ 
where the last estimate is easy to see using Riemann sums for $\int h(t)dt$.
We observe that $\Phi_n^2$ is a trigonometric polynomial, and it is clear that 
$$
|\Phi_n(t)|^2\le \Phi_n(0)^2.
$$
Consequently, \eqref{eq:kernupbd} follows from the definition of $\Psi_n$. 
The estimate \eqref{eq:kernbern} is easy to deduce from the fact that $\|(\Phi_n^2)'\|\ls n^3$, so that
$$
\ba
|\Psi_n(x,y)-\Psi_n(z,w)|&\le |\Phi_n^2(\rho(x,y))-\Phi_n^2(\rho(z,w))|\le |\Phi_n^2(\rho(x,y))-\Phi_n^2(\rho(z,y))|+|\Phi_n^2(\rho(z,y))-\Phi_n^2(\rho(z,w))|\\
&\ls n^3\left\{|\rho(x,y)-\rho(z,y)|+|\rho(z,y)-
\rho(z,w)|\right\}\ls n^3\left\{\rho(x,z)+\rho(y,w)\right\}.
\ea
$$
The estimate \eqref{eq:kernlowbd} follows from \eqref{eq:triglowbd} and the definition of $\Psi_n$.
\end{proof}
 
\noindent\textit{Proof of Lemma~\ref{lem:integralest}.} We proceed by examining concentric annuli. Let $x\in \mathbb{M}$ be fixed, and set $A_0=\mathbb{B}(x,d)$ and $A_k=\mathbb{B}(x,2^kd)\setminus \mathbb{B}(x,2^{k-1}d)$ for every $k\geq 1$. First suppose that $nd\geq 1$. Then by \eqref{eq:metrickernloc}~and~\eqref{eq:ballmeasurecon}, we deduce
\be\label{eq:annuli}
\ba
\int_{\mathbb{M}\setminus \mathbb{B}(x,d)}\Psi_n(x,y)d\mu(y)=&\sum_{k=1}^\infty\int_{A_k}\Psi_n(x,y)d\mu(y)
\lesssim\sum_{k=1}^\infty \frac{\mu(A_k)n^2}{\max(1,2^{k-1}dn)^S}\\
&\lesssim\sum_{k=1}^\infty \frac{2^{k\alpha} d^\alpha n^2}{2^{S(k-1)}(dn)^S}
\lesssim n^{2-\alpha}(nd)^{\alpha-S}\sum_{k=1}^\infty 2^{k(\alpha-S)}
\lesssim  n^{2-\alpha}(nd)^{\alpha-S}.
\ea
\ee
If $nd=1$, we observe
\be\label{eq:annuli2}
\int_{A_0}\Phi_{n}(\rho(x,y))^2d\mu(y)\lesssim \mu(A_0)n^2\lesssim d^{\alpha}n^2=n^{2-\alpha}.
\ee
Combining \eqref{eq:annuli}~and~\eqref{eq:annuli2} when $nd=1$ yields \eqref{eq:In}. When $dn\leq 1$, we see 
\be
\int_{\mathbb{M}\setminus\mathbb{B}(x,d)}\Psi_n(x,y)d\mu(y)\leq I_n\lesssim n^{2-\alpha}.
\ee
Together with \eqref{eq:annuli}, this completes the proof of \eqref{eq:Kn}. 
There is no loss of generality in using the same constant $C_1$ in both of these estimates.
We see, in view of \eqref{eq:kernupbd}, \eqref{eq:kernlowbd}, and the detectability of $\mu$, that if $x\in\mathbb{X}$ it follows that
\be
\int_{\mathbb{X}} \Psi_n(x,y)d\mu(y)\gtrsim \int_{\mathbb{B}(x,\pi/(6n))} n^2 d\mu(y)\gtrsim n^{2-\alpha},
\ee
demonstrating \eqref{eq:Jn} and completing the proof.
\qed

\begin{theorem}\label{thm:suppmu}
Let $\mu$ be detectable and $S>\alpha$. 
If $\theta\leq C_2/(4C_1)$, then by setting
\be\label{eq:dtheta}
d(\theta)=\left(\frac{C_1}{C_2\theta}\right)^{1/(S-\alpha)},
\ee
it follows that (cf. \eqref{eq:Sn})
\be\label{eq:thm1inc}
\mathbb{X}\subseteq \mathcal{S}_{n}(\theta)\subseteq \mathbb{B}(\mathbb{X},d(\theta)/n).
\ee
\end{theorem}

\begin{proof}
From \eqref{eq:In}~and~\eqref{eq:Jn}, we see that for any $x\in\mathbb{X}$,
\be\label{eq:Jnbound}
\sigma_n(x)\geq J_n \geq \frac{C_2I_n}{C_1}.
\ee
With our assumption of $\theta\leq C_2/(4C_1)$, this proves the inclusion
\be
\mathbb{X}\subseteq \mathcal{S}_n(\theta).
\ee
Note that $C_1^2/C_2^2\ge 1 >1/4$, so that $\theta\leq C_2/(4C_1)< C_1/C_2$, and hence, $d(\theta)>1$. Then, for any $x\in \mathbb{M}$ such that $\dist(x,\mathbb{X})\geq d(\theta)/n$, we have by \eqref{eq:Kn} that
\be
\sigma_n(x)\leq \int_{\mathbb{M}\setminus \mathbb{B}(x,d(\theta)/n)}\Psi_n(x,y)d\mu(y)\leq C_1n^{2-\alpha}/d(\theta)^{S-\alpha}\leq \theta C_2n^{2-\alpha}\leq \theta I_n.
\ee
This demonstrates the inclusion
\be
\mathcal{S}_n(\theta)\subseteq \mathbb{B}(\mathbb{X},d(\theta)/n),
\ee
completing the proof.
\end{proof}

\begin{theorem}\label{thm:partmu}
Assume the setup of Theorem~\ref{thm:suppmu} and suppose $\mu$ has a fine structure. Define
\be\label{eq:Skn}
\mathcal{S}_{k,\eta,n}(\theta)\coloneqq \mathcal{S}_n(\theta)\cap \mathbb{B}(\mathbf{S}_{k,\eta},d(\theta)/n).
\ee
Let $n\geq 2d(\theta)/\eta$, $\mu(\mathbf{S}_{K_{\eta}+1,\eta})\leq \frac{C_2}{C_0}\theta n^{-\alpha}$, and $j,k=1,\dots,K_\eta$ with $j\neq k$.  Then 
\be\label{eq:Snpartition}
\mathcal{S}_n(\theta)=\bigcup_{k=1}^{K_\eta} \mathcal{S}_{k,\eta,n}(\theta)
\ee
and,
\be\label{eq:minsep2}
\operatorname{dist}(\mathcal{S}_{j,\eta,n}(\theta),\mathcal{S}_{k,\eta,n}(\theta))\geq \eta.
\ee
Furthermore,
\be\label{eq:thm2inc}
\mathbb{X}\cap \mathbb{B}(\mathbf{S}_{k,\eta},d(\theta)/n)\subseteq \mathcal{S}_{k,\eta,n}(\theta)\subseteq \mathbb{B}(\mathbf{S}_{k,\eta},d(\theta)/n).
\ee
\end{theorem}

\begin{proof}
The first inclusion in \eqref{eq:thm2inc} is satisfied from \eqref{eq:thm1inc} and the second is satisfied by the definition of $\mathcal{S}_{k,\eta,n}$. In view of the assumption that $\eta\geq 2d(\theta)/n$ and Definition~\ref{def:finestructure}, we see that 
\be
\dist(\mathbb{B}(\mathbf{S}_{j,\eta},d(\theta)/n),\mathbb{B}(\mathbf{S}_{k,\eta},d(\theta)/n))\geq \eta,
\ee
for any $j\neq k$. Since $\mathcal{S}_{k,\eta,n}(\theta)\subseteq \mathbb{B}(\mathbf{S}_{k,\eta},d(\theta)/n)$, it follows that the separation condition \eqref{eq:minsep2} must also be satisfied.
Now it remains to show \eqref{eq:Snpartition}.
Let us define, in this proof only,
\be
\mathbf{S}=\bigcup_{k=1}^{K_\eta}\mathbf{S}_{k,\eta}.
\ee
It is clear from \eqref{eq:Skn} that $\bigcup_{k=1}^{K_\eta} \mathcal{S}_{k,\eta,n}(\theta)\subseteq\mathcal{S}_n(\theta)$. 
We note that for any $x\in\mathbb{M}\setminus \mathbb{B}(\mathbf{S},d(\theta)/n)$, we have $\dist(x,\mathbf{S})\geq d(\theta)/n$ and as a result
\be
\ba
\sigma_n(x)=& \int_{\mathbf{S}_{K_\eta+1}}\Psi_n(x,y)d\mu(y)+\int_{\mathbf{S}}\Psi_n(x,y)d\mu(y)\\
\leq& C_0n^2\mu(\mathbf{S}_{K_{\eta}+1})+\int_{\mathbf{S}\setminus \mathbb{B}(x,d(\theta)/n)}\Psi_n(x,y)d\mu(y)&\text{(By \eqref{eq:kernupbd})}\\
\leq&C_2n^{2-\alpha}\theta+C_1n^{2-\alpha}d(\theta)^{\alpha-S}&\text{(By the assumption on $\mu(\mathbf{S}_{K_\eta+1})$ and \eqref{eq:Kn})}\\
\leq& 2C_2n^{2-\alpha}\theta \le 2J_n\theta &\text{(By \eqref{eq:dtheta})}\\
\leq&2\theta I_n.
\ea
\ee
Thus, $x\notin \mathcal{S}_n(\theta)$ and, equivalently, $\mathcal{S}_n(\theta)\subseteq \mathbb{B}(\mathbf{S},d(\theta)/n)$. Therefore we have shown $\mathcal{S}_n(\theta)=\mathbb{B}(\mathbf{S},d(\theta)/n)$, completing the proof.
\end{proof}

\subsection{Discretization}
\label{sec:discreteproofs}

In this section we relate the continuous support estimator and estimation sets to the discrete cases based on randomly sampled data. The conclusion of this section will be the proofs to the theorems from Section~\ref{sec:mainresults}. To aid us in this process we first state a consequence of the Bernstein Concentration inequality as a proposition.

\begin{proposition}\label{prop:concentration}
Let $X_1,\cdots, X_M$ be independent real valued random variables such that for each $j=1,\cdots,M$, $|X_j|\le R$, and $\mathbb{E}(X_j^2)\le V$. Then for any $t>0$,
\be\label{bernstein_concentration}
\mathsf{Prob}\left( \left|\frac{1}{M}\sum_{j=1}^M (X_j-\mathbb{E}(X_j))\right| \ge Vt/R\right) \le 2\exp\left(-\frac{MVt^2}{2R^2(1+t)}\right).
\ee
\end{proposition}


%

Information and a derivation for the Bernstein concentration inequality are standard among many texts in probability; we list \cite[Section~2.1, 2.7]{concineq} as a reference. 
It is instinctive to use Proposition~\ref{prop:concentration} with $X_j=\Psi_n(x,x_j)$. 
This would yield the desired bound for any value of $x\in\MM$. 
However, to get an estimate on the supremum norm of the difference $F_n-\sigma_n$, we need to find an appropriate net for $\MM$ (and estimate its size) so that the point where this supremum is attained is within the right ball around one of the points of the net. 
Usually, this is done via a Bernstein inequality for the gradients of the objects involved.
In the absence of any differentiability structure on $\MM$, we need a more elaborate argument.

The following proposition is a consequence of \cite[Theorem~7.2]{mhaskar2020kernel}, and asserts the existence of a partition of $\mathbb{X}$ satisfying properties which will be helpful to proving our main results.

\begin{proposition}\label{prop:partition} 
Let $\delta>0$. There exists a partition $\{Y_k\}_{k=1}^N$ of $\XX$ such that for each $k$, $\mathsf{diam}(Y_k)\le 36\delta$, and $\mu(Y_k)\sim \delta^\alpha$. 
In particular, $N\ls \delta^{-\alpha}$.
\end{proposition}

Recall that we denote our data by $\mathcal{D}=\{x_j\}_{j=1}^M$, where each $x_j$ is sampled uniformly at random from $\mu$. In the sequel, we let  $\D_k=\D\cap Y_k$, $k=1,\cdots, N$.
The following lemma gives an estimate for $|\D_k|$.
\begin{lemma}\label{lemma:discpartition}
Let $0<\delta, \epsilon, t<1$. 
If 
\be\label{eq:Mcond}
M\gs t^{-2}\delta^{-\alpha}\log(c/(\epsilon\delta^\alpha)),
\ee  
then
\be\label{eq:discpartition}
\mathsf{Prob}\left(\max_{1\le k\le N}\left|\frac{\mu(Y_k)M}{|\D_k|}-1\right|\ge t\right)\ls \epsilon.
\ee
\end{lemma}

\begin{proof}\ 
Let $k$ be fixed, and in this proof only,  $\chi_k$ denotes the characteristic function of $Y_k$. 
Thought of as a random variable, it is clear that $|\chi_k|\le 1$, $\int \chi_k(z)d\mu(z)=\int \chi_k(z)^2d\mu(z)=\mu(Y_k)$. 
Moreover, $\sum_{j=1}^M \chi_k(z_j)=|\D_k|$.
So, we may apply Proposition~\ref{prop:concentration}, and recall that $\mu(Y_k)\sim\delta^\alpha$ to conclude that
\be\label{eq:pf1eqn1}
\begin{aligned}
\mathsf{Prob}\left(\left|\frac{|\D_k|}{M}-\mu(Y_k)\right|\ge \frac{\mu(Y_k)t}{1+t}\right) &= \mathsf{Prob}\left(\left|\frac{|\D_k|}{M\mu(Y_k)}-1\right|\ge \frac{t}{1+t}\right)\\
&\le 2\exp\left(-\frac{M\mu(Y_k)t^2}{(1+t)(1+2t)}\right)\le 2\exp\left(-cM\delta^\alpha t^2\right).
\end{aligned}
\ee
(In the last estimate, we have used the fact that for $0<t<1$, $(1+t)(1+2t)\sim 1$.)
Next, we observe that
$$
\left|\frac{M\mu(Y_k)}{|\D_k|}-1\right|=\frac{M\mu(Y_k)}{|\D_k|}\left|\frac{|\D_k|}{M\mu(Y_k)}-1\right|.
$$
\vskip 2pt
So, if $\disp \left|\frac{|\D_k|}{M\mu(Y_k)}-1\right|< \frac{t}{1+t}$, then $\disp \frac{|\D_k|}{M\mu(Y_k)}\ge 1/(1+t)$, and hence,
$\disp \left|\frac{M\mu(Y_k)}{|\D_k|}-1\right|<t$.
Thus, for every $k$,
\be\label{eq:pf1eqn2}
\mathsf{Prob}\left(\left|\frac{M\mu(Y_k)}{|\D_k|}-1\right|\ge t\right) \le \mathsf{Prob}\left(\left|\frac{|\D_k|}{M\mu(Y_k)}-1\right|\ge \frac{t}{1+t}\right)\le 2\exp\left(-cM\delta^\alpha t^2\right).
\ee
Since the number of elements $Y_k$ in the partition is $\ls \delta^{-\alpha}$,
we conclude that
\be\label{eq:pf1eqn3}
\mathsf{Prob}\left(\max_{k\in [N]}\left|\frac{M\mu(Y_k)}{|\D_k|}-1\right|\ge t\right)\ls \delta^{-\alpha}\exp\left(-cM\delta^\alpha t^2\right).
\ee
We set the right hand side of the above inequality to $\epsilon$ and solve for $M$ to complete the proof.
\end{proof}

In order to prove the bounds we want in Lemma~\ref{lem:Inbounds}, we rely on a function which estimates both our discrete and continuous measure support estimators $F_n$ and $\sigma_n$. We define this function as
\be
H_n(x)\coloneqq \sum_{k=1}^N \frac{\mu(Y_k)}{|\D_k|}\sum_{x_j\in\D_k}\Psi_n(x,x_j).
\ee
The following lemma relates this function to our continuous measure support estimator.

\begin{lemma}\label{lemma:discprelim}
Let $0<\gamma<2$, $n\ge 2$. 
There exists a constant $c(\gamma)$ with the following property. 
Suppose $0<\delta\le c(\gamma)/n$, $\{Y_k\}$ be a partition of $\XX$ as in Proposition~\ref{prop:partition}, and we continue the notation before. 
We have
\be\label{eq:discprelim}
\max_{x\in\mathbb{M}}\left|H_n(x) -\sigma_n(x)\right| \le (\gamma/2) I_n.
\ee
\end{lemma}
\begin{proof}\ 
In this proof, all the constants denoted by $c_1,c_2,\cdots$ will retain their values.
Let $x\in\mathbb{M}$.
We will fix $\delta$ to be chosen later.
Also, let $r\ge \delta$ be a parameter to be chosen later, $\mathcal{N}=\{k : \mathsf{dist}(x, Y_k) <r\}$, $\mathcal{L}=\{k : \mathsf{dist}(x, Y_k) \ge r\}$ and for $j=0,1,\cdots$, $\mathcal{L}_j=\{k :  2^jr\le \mathsf{dist}(x, Y_k) <2^{j+1} r\}$. 

In view of \eqref{eq:kernbern}, we have for $k\in \mathcal{N}$ and $x_j\in \D_k$,
$$
\left|\mu(Y_k)\Psi_n(x,x_j)-\int_{Y_k}\Psi_n(x,y)d\mu(y)\right|\le \int_{Y_k}\left|\Psi_n(x,x_j)-\Psi_n(x,y)\right|d\mu(y)
\ls n^3\int_{Y_k}\rho(z,y)d\mu(y) \le c_1n^3\mathsf{diam}(Y_k)\mu(Y_k).
$$
Consequenty, for $k\in \mathcal{N}$,
\be\label{eq:pf2eqn1}
\left|\frac{\mu(Y_k)}{|\D_k|}\sum_{x_j\in\D_k}\Psi_n(x,x_j) -\int_{Y_k}\Psi_n(x,y)d\mu(y)\right| \le c_1n^3\mathsf{diam}(Y_k)\mu(Y_k).
\ee
Since $\cup_{k\in\mathcal{N}}Y_k\subseteq \BB(x,r)$, we have 
\be\label{eq:pf2eqn5}
\sum_{k\in\mathcal{N}}\mu(Y_k) =\mu\left(\cup_{k\in\mathcal{N}}Y_k\right)\le \mu(\BB(x,r))\ls r^\alpha.
\ee
We deduce from \eqref{eq:pf2eqn1}, \eqref{eq:pf2eqn5} and the fact that $\mathsf{diam}(Y_k)\ls \delta$ that
\be\label{eq:pf2eqn2}
\left|\sum_{k\in\mathcal{N}}\left(\frac{\mu(Y_k)}{|\D_k|}\sum_{x_j\in\D_k}\Psi_n(x,x_j) -\int_{Y_k}\Psi_n(x,y)d\mu(y)\right)\right| \le 
c_3n^3\delta r^\alpha = c_3(n\delta) (nr)^\alpha n^{2-\alpha}. 
\ee
Next, let $k\in \mathcal{L}_j$ for some $j\ge 0$. 
Then the localization estimate \eqref{eq:metrickernloc} shows that for any $x_j\in \D_k$,
$$
\left|\mu(Y_k)\Psi_n(x,x_j)-\int_{Y_k}\Psi_n(x,y)d\mu(y)\right|\le c_4n^2 (2^jnr)^{-S}\mu(Y_k),
$$
so that
\be\label{eq:pf2eqn3}
\left|\frac{\mu(Y_k)}{|\D_k|}\sum_{x_j\in \D_k}\Psi_n(x,x_j) -\int_{Y_k}\Psi_n(x,y)d\mu(y)\right|\le c_4n^2\mu(Y_k) (2^jnr)^{-S}.
\ee
Arguing as in the derivation of \eqref{eq:pf2eqn5}, we deduce that 
$$
\mu\left(\cup_{k\in \mathcal{L}_j}Y_k\right)\ls (2^jr)^\alpha.
$$
 Since $S>\alpha$, we deduce that
if $r\ge 1/n$, then
\be\label{eq:pf2eqn4}
\begin{aligned}
\left|\sum_{k\in \mathcal{L}} \frac{\mu(Y_k)}{|\D_k|}\sum_{x_j\in \D_k}\Psi_n(x,x_j) -\int_{Y_k}\Psi_n(x,y)d\mu(y)\right|&\le \sum_{j=0}^\infty \sum_{k\in \mathcal{L}_j} \left|\frac{\mu(Y_k)}{|\D_k|}\sum_{x_j\in \D_k}\Psi_n(x,x_j) -\int_{Y_k}\Psi_n(x,y)d\mu(y)\right|\\
&\le c_4n^{2-\alpha}(nr)^{\alpha-S}\sum_{j=0}^\infty 2^{j(\alpha-S)} \le c_5n^{2-\alpha}(nr)^{\alpha-S}.
\end{aligned}
\ee
Since $S>\alpha$, we may choose $r\sim_\gamma n$ such that $c_5 (nr)^{\alpha-S}\le \gamma/4$, and then require $\delta\le \min(r, c_6(\gamma)/n)$ so that in \eqref{eq:pf2eqn2}, $c_3(nr)^\alpha n\delta\le \gamma/4$.
Then, recalling that (cf. \eqref{eq:In}) $I_n\sim n^{2-\alpha}$, \eqref{eq:pf2eqn4} and \eqref{eq:pf2eqn2} lead to \eqref{eq:discprelim}.
\end{proof}

\begin{remark}
We note that the proof of Lemma~\ref{lemma:discprelim} involves only the upper bound of the detectability condition (Definition~\ref{def:detectable}). The lower detectability bound instead plays a role in deducing the lower bound on $J_n$ (see \eqref{eq:Jnbound}) which in turn aids in deducing lower bounds in Theorems~\ref{thm:suppmu}~and~\ref{thm:partmu}.
\end{remark}

In the following lemma, we establish a connection between the sum $F_n$ as defined in \eqref{eq:support_estimator_def} and the value $I_n$ from Section~\ref{sec:measureproofs}. Since we have already established the bound between $H_n$ and $\sigma_n$, we focus on the bound between $H_n$ and $F_n$ in this lemma.

\begin{lemma}\label{lem:Inbounds}
Let $n\ge 2$, $0<\beta<2$, $M\gs_\beta n^\alpha\log n$, and $\mathcal{D}=\{x_j\}_{j=1}^M$ be independent random samples from a detectable measure $\mu$. 
Then with probability $\ge 1-1/n$, we have for all $x\in \mathbb{M}$,
\be\label{eq:betacomp1}
 |F_n(x)-\sigma_n(x)|\leq \beta I_n.
\ee
Consequently,
\be\label{eq:betacomp2}
(1-\beta)I_n\leq \max_{x\in\mathbb{M}} F_n(x)\leq (1+\beta)I_n.
\ee
\end{lemma}

\begin{proof}\ 
Let $\gamma\in (0,1)$ to be chosen later, and $\{Y_k\}$ be a partition as in Lemma~\ref{lemma:discprelim}.
In view of Lemma~\ref{lemma:discprelim}, we see that
\be\label{eq:pf3eqn1}
(1-\gamma/2)I_n\le \max_{x\in\mathbb{M}} H_n(x)\le (1+\gamma/2) I_n.
\ee
In view of Lemma~\ref{lemma:discpartition}, we see that for $M\ge c(\gamma)n^\alpha\log n$, we have with probability $\ge 1-1/n$,
\be\label{eq:pf3eqn2}
\max_k \left|\frac{\mu(Y_k)M}{|\D_k|}-1\right|\le \gamma/2.
\ee
In particular, 
\be
1-\gamma/2\le \frac{\mu(Y_k)M}{|\D_k|}\le 1+\gamma/2.
\ee
Hence, \eqref{eq:pf3eqn1} leads to
\be\label{eq:pf3eqn3}
\max_{x\in\mathbb{M}}F_n(x)\le \frac{2+\gamma}{2-\gamma}I_n.
\ee
Using \eqref{eq:pf3eqn2} again, we see that
\be
\left|F_n(x)-H_n(x)\right| \le (\gamma/2)\max_{x\in\mathbb{M}}F_n(x)\le (\gamma/2)\frac{2+\gamma}{2-\gamma}I_n.
\ee
Together with \eqref{eq:discprelim}, this implies that
\be\label{eq:pf3eqn4}
\left|F_n(x)-\sigma_n(x)\right| \le (\gamma/2)\left(1+\frac{2+\gamma}{2-\gamma}\right)I_n=\frac{2\gamma}{4-\gamma}I_n
\ee
We now choose $\gamma=4\beta/(2+\beta)$, so that the right hand side of \eqref{eq:pf3eqn4} is $\beta I_n$. We can verify $0<\gamma<2$ whenever $\beta<2$.
\end{proof}

The following lemma gives bounds on $\max_{k\in [M]}F_n(x_k)$, which is a crucial component to the proof involving our finite data support estimation set $\mathcal{G}_n(\Theta)$. We note briefly the critical difference between Lemma~\ref{lem:Inbounds} and Lemma~\ref{lem:Gnboundsprep} is that the former considers the maximum of $F_n$ over the entire metric space $\MM$, while the latter considers the maximum over the finite set of data points which are sampled from the measure $\mu$.

\begin{lemma}
\label{lem:Gnboundsprep}
Let $\mathcal{D}=\{x_j\}_{j=1}^M$ be independent random samples from a detectable measure $\mu$. If $0<\beta<C_2/C_1$ and $M\gtrsim_\beta n^\alpha\log(n)$, with  probability $\geq 1-1/n$ we have
\be\label{eq:Fnxkbound}
\left(\frac{C_2}{C_1}-\beta\right)I_n\leq \max_{k\in [M]}F_n(x_k)\leq (1+\beta)I_n.
\ee
\end{lemma}

\begin{proof}
Necessarily, $\mathcal{D}\subseteq \XX$.
Using \eqref{eq:Jn}, we deduce that
\be
\max_{k\in [M]}F_n(x)\geq \max_{k\in [M]}\sigma_n(x_k)-\beta I_n\geq J_n -\beta I_n\ge C_2n^{2-\alpha}-\beta I_n \geq \left(\frac{C_2}{C_1}-\beta\right)I_n.
\ee
This proves \eqref{eq:Fnxkbound}.
The second inequality is satisfied by \eqref{eq:betacomp2} since $\max_{k\in[M]}F_n(x_k)\leq \max_{x\in\mathbb{M}}F_n(x)$. 
\end{proof}

Now with the prepared bounds on $\max_{k\in [M]}F_n(x_k)$, we give a theorem from which Theorems~\ref{thm:fullsuportdet}~and~\ref{thm:class_separation} are direct consequences.
In the sequel, we will denote $C_2/C_1$ by $C^*$. \yadi{$C^*$}{$C_2/C_1$}

\begin{theorem}\label{thm:GnSnbound}
Let $n\geq 2$, and $\mathcal{D}=\{x_j\}_{j=1}^M$ be sampled from the detectable probability measure $\mu$.
Let $0<\Theta< 1$.
 If $M\gtrsim n^\alpha\log(n)$ then with probability at least $1-c_1M^{-c_2}$ we have
\be\label{eq:SnGncomp}
\mathcal{S}_n\left(\frac{(1+C^*)\Theta}{4}\right)\subseteq \mathcal{G}_n(\Theta, \mathcal{D})\subseteq\mathcal{S}_n(C^*\Theta/8).
\ee
\end{theorem}

\begin{proof}
In this proof, we will invoke Lemma~\ref{lem:Inbounds}~and~\ref{lem:Gnboundsprep} with $\beta=C^*\Theta/2$. For this proof only, define
\be
t=\frac{2\Theta+C^*\Theta(1+\Theta)}{8}=\frac{(1+\beta)\Theta+\beta}{4},
\ee
and suppose $x\in\mathcal{S}_n(t)$; i.e., $\sigma_n(x)\ge 4tI_n$.
 Then with probability $1-c_1M^{-c_2}$ we have
\be
\ba
\Theta\max_{k\in[M]}F_n(x_k)=&\frac{4t-\beta}{1+\beta}\max_{k\in[M]}F_n(x_k)\\
\leq&(4t-\beta)I_n &\text{(From \eqref{eq:Fnxkbound})}\\
\leq&\sigma_n(x)-\beta I_n &\text{(From \eqref{eq:Sn})}\\
\leq& F_n(x). &\text{(From \eqref{eq:betacomp1})}
\ea
\ee
This results in the first inclusion in \eqref{eq:SnGncomp} because $\Theta\le 1$ implies that
$$
\frac{(1+C^*)\Theta}{4}\geq \frac{2\Theta+C^*\Theta(1+\Theta)}{8},
$$
and hence,
\be
\mathcal{S}_n\left(\frac{(1+C^*)\Theta}{4}\right)\subseteq \mathcal{S}_n\left(\frac{2\Theta+C^*\Theta(1+\Theta)}{8}\right)\subseteq\mathcal{G}_n(\Theta,\mathcal{D}).
\ee
Now suppose $x\in\mathcal{G}(\Theta,\mathcal{D})$. Then
\be
\ba
4(C^*\Theta/8)I_n=& (C^*\Theta-\beta)I_n\\
\leq& \Theta \max_{k\in[M]}F_n(x_k)-\beta I_n&\text{(From \eqref{eq:Fnxkbound})}\\
\leq& F_n(x)-\beta I_n&\text{(From \eqref{eq:support_est_set_def})}\\
\leq& \sigma_n(x)&\text{(From \eqref{eq:betacomp1})}.
\ea
\ee
This gives us the second inclusion in \eqref{eq:SnGncomp}, completing the proof.
\end{proof}

We are now prepared to give the proofs of our main results from Section~\ref{sec:mainresults}. A key element of both proofs is to utilize Theorem~\ref{thm:GnSnbound} in conjunction with the corresponding results on the continuous support estimation set $\mathcal{S}_n(\theta)$ (Theorems~\ref{thm:suppmu}~and~\ref{thm:partmu}).\\

\noindent\textit{Proof of Theorem~\ref{thm:fullsuportdet}.} By \eqref{eq:thm1inc} and \eqref{eq:SnGncomp} with $\theta_{1}=\frac{(1+C^*)\Theta}{4}$, we have
\be\label{eq:lowerinclusion}
\mathbb{X}\subseteq \mathcal{S}_{n}(\theta_{1})\subseteq \mathcal{G}_{n}(\Theta).
\ee
Similarly, with $\theta_{2}=\frac{C^*\Theta}{8}$, and the definition of $d(\theta_{2})=\left( \frac{C_{1}}{C_{2}\theta_{2}} \right)^{1/(S-\alpha)}$ (from \eqref{eq:dtheta}), we see
\be\label{eq:gnrightinc}
\mathcal{G}_{n}(\Theta)\subseteq \mathcal{S}_{n}(\theta_{2})\subseteq \mathbb{B}\left( \mathbb{X},\left( \frac{C_{1}}{C_{2}\theta_2} \right)^{1/(S-\alpha)}\Big/n \right)=\mathbb{B}\left( \mathbb{X},\left( \frac{8C_{1}}{C^*C_{2}\Theta} \right)^{1/(S-\alpha)}\Big/ n\right).
\ee
The choice of $\theta$ in each case satisfies the conditions of Theorem~\ref{thm:suppmu} because
\be
\theta_{2}=\frac{C^*\Theta}{8}\leq \theta_{1}=\frac{(1+C^*)\Theta}{4}< \frac{C_{2}}{4C_{1}}.
\ee
Setting
\be\label{eq:rtheta}
r(\Theta):= \left( \frac{8C_{1}}{C^*C_{2}\Theta} \right)^{1/(S-\alpha)},
\ee
then \eqref{eq:gnrightinc} demonstrates \eqref{eq:thm1}.\qed\\

\noindent\textit{Proof of Theorem~\ref{thm:class_separation}.} We note that the inclusion $\mathcal{G}_{k,\eta,n}(\Theta)\subseteq \mathbb{B}(\mathbf{S}_{k,\eta},r(\Theta)/n)$ is already satisfied by \eqref{eq:thm2-1}. Let $\theta_1=\frac{(1+C^*)\Theta}{4}$, and observe that
\be
d(\theta_1)=\left(\frac{4C_1}{C_2(1+C^*)\Theta}\right)^{1/(S-\alpha)}=\left(\frac{C^*}{2(1+C^*)}\right)^{1/(S-\alpha)}r(\Theta),
\ee
as defined in \eqref{eq:rtheta}. We set $c=\left(\frac{C^*}{2(1+C^*)}\right)^{1/(S-\alpha)}$ and note $c<1$. Therefore,
\be
\ba
\mathbb{X}\cap \mathbb{B}(\mathbf{S}_{k,\eta},cr(\Theta)/n)=&\mathbb{X}\cap \mathbb{B}(\mathbf{S}_{k,\eta},d(\theta_1)/n)&\\
\subseteq& \mathcal{S}_n(\theta_1)\cap \mathbb{B}(\mathbf{S}_{k,\eta},d(\theta_1)/n)&\text{(From \eqref{eq:thm2inc})}\\
\subseteq& \mathcal{S}_n(\theta_1)\cap \mathbb{B}(\mathbf{S}_{k,\eta},r(\Theta)/n)&\text{(Since $d(\theta_1)<r(\Theta)$)}\\
\subseteq& \mathcal{G}_n(\Theta)\cap \mathbb{B}(\mathbf{S}_{k,\eta},r(\Theta)/n)&\text{(From \eqref{eq:SnGncomp})}\\
=&\mathcal{G}_{k,\eta,n}(\Theta),
\ea
\ee
completing the proof of \eqref{eq:thm2-3}.
Setting $\theta_2=\frac{C^*\Theta}{8}$, then $r(\Theta)=d(\theta_2)$ and \eqref{eq:SnGncomp} gives us the inclusion
\be
\mathcal{G}_{k,\eta,n}(\Theta)=\mathcal{G}_n(\Theta)\cap \mathbb{B}(\mathbf{S}_{k,\eta},r(\Theta)/n)\subseteq \mathcal{S}_n(\theta_2)\cap \mathbb{B}(\mathbf{S}_{k,\eta},d(\theta_2)/n)=\mathcal{S}_{k,\eta,n}(\theta_2).
\ee
Then, \eqref{eq:minsep2} implies \eqref{eq:thm2-2}.\qed


\subsection{F-score proof}

In this section we give a proof for Theorem~\ref{thm:fscore}.

\begin{proof}
Observe that under the conditions of Theorem~\ref{thm:class_separation} we have
\be
\mu(\mathbf{S}_{k,\eta})\leq \mu(\mathcal{G}_{k,\eta,n})\leq \mu(\mathbf{S}_{k,\eta})+\mu(\mathbf{S}_{K_\eta+1,\eta}).
\ee
Therefore,
\be
\mathcal{F}_\eta(\mathcal{G}_{k,\eta,n})\geq 2\frac{\mu(\mathbf{S}_{k,\eta})}{2\mu(\mathbf{S}_{k,\eta})+\mu(\mathbf{S}_{K_{\eta}+1,\eta})}\geq 1-\frac{\mu(\mathbf{S}_{K_{\eta}+1,\eta})}{2\mu(\mathbf{S}_{k,\eta})+\mu(\mathbf{S}_{K_{\eta}+1,\eta})}.
\ee
Also, $\{\mathcal{G}_{k,\eta,n}\}_{k=1}^{K_\eta}$ is a partition of $\mathcal{G}_n(\Theta)\supseteq \mathbb{X}$. Then,
\be
\sum_{k=1}^{K_\eta}\mu(\mathcal{G}_{k,\eta,n})\mathcal{F}_\eta(\mathcal{G}_{k,\eta,n})\geq 1-\mu(\mathbf{S}_{K_{\eta}+1,\eta})\sum_{k=1}^{K_\eta}\frac{\mu(\mathcal{G}_{k,\eta,n})}{2\mu(\mathbf{S}_{k,\eta})+\mu(\mathbf{S}_{K_\eta+1,\eta})},
\ee
further implying
\be
1\geq \mathcal{F}_\eta\left(\{\mathcal{G}_{k,\eta,n}\}_{k=1}^{K_\eta}\right)\geq 1-\frac{\mu(\mathbf{S}_{K_{\eta}+1,\eta})}{2\min_{k\in \{1,\dots,K_\eta\}}\mu(\mathbf{S}_{k,\eta})}.
\ee
Thus, by our assumption
\be
\lim_{\eta\to 0^+}\mathcal{F}_\eta\left(\{\mathcal{G}_{k,\eta,n}\}_{k=1}^{K_\eta}\right)\leq \lim_{\eta\to 0^+} 1-\frac{\mu(\mathbf{S}_{K_{\eta}+1,\eta})}{2\min_{k\in \{1,\dots,K_\eta\}}\mu(\mathbf{S}_{k,\eta})}=1.
\ee
\end{proof}

\section{Conclusions}

In this paper, we have introduced a new approach for active learning in the context of machine learning. 
We provide theory for measure support estimation based on finitely many samples from an unknown probability measure supported on a compact metric space. 
With an additional assumption on the measure, known as the fine structure, we then relate this theory to the classification problem, which can be viewed as estimating the supports of a measure of the form $\mu=\sum_{k=1}^K c_k\mu_k$. We have shown that this setup is a generalization of the well-known signal separation problem. Therefore our theory unifies ideas from signal separation with machine learning classification. Since the measures we are considering may be supported on a continuum, our theory additionally relates to the super-resolution regime of the signal separation problem. 

We also give some empirical analysis for the performance of our new algorithm MASC, which was originally introduced in a varied form in \cite{mhaskar-odowd-tsoukanis}. The key focus of the algorithm is on querying high-information points whose labels can be extended to others belonging to the same class with high probability. This is done in a multiscale manner, with the intention to be applied to data sets where the minimal separation between supports of the measures for different classes may be unknown or even zero. We applied MASC to a document data set as well as two hyperspectral data sets, namely subsets of the Indian Pines and Salinas hyperspectral imaging data sets. In the process of these experiments, we demonstrate empirically that MASC is selecting high-information points to query and that it gives competitive performance compared to two other recent active learning methods: LAND and LEND. Specifically, MASC consistently outperforms these algorithms in terms of computation time and exhibits competitive accuracy on Indian Pines for a broad range of query budgets.

 \newpage

\bibliographystyle{abbrv}
\bibliography{refs.bib}

\end{document}